\documentclass[twoside]{article}
\usepackage{ecj,palatino,epsfig,latexsym,natbib}
\usepackage{pgfplots} 
\pgfplotsset{compat=1.18}
\usetikzlibrary{arrows}

\usepackage{xcolor} 
\usepackage{amsmath,amsfonts,amsthm,amssymb}
\usepackage{enumerate} 
\usepackage{booktabs}
\usepackage{algorithmic}
\usepackage{algorithm} 
\newtheorem{definition}{Definition}
\newtheorem{lemma}{Lemma} 
\newtheorem{theorem}{Theorem}  
\newtheorem{proposition}{Proposition}   
\newtheorem{corollary}{Corollary}

\begin{document}

\ecjHeader{x}{x}{xxx-xxx}{201X}{Drift Analysis with Fitness Levels}{J. He and Y. Zhou}
\title{\bf Drift Analysis with Fitness Levels for  Elitist  Evolutionary Algorithms}  

\author{\name{\bf Jun He} \hfill \addr{jun.he@ntu.ac.uk}\\ 
        \addr{Department of Computer Science, Nottingham Trent University, Clifton Campus, Nottingham NG11 8NS, UK}
\AND
       \name{\bf Yuren Zhou} \hfill  \addr{zhouyuren@mail.sysu.edu.cn}
       \\
        \addr{School of Data and Computer Science, Sun Yatsen University,  {Guangzhou}, {510006}, 
          {China}}
}

\maketitle

\begin{abstract} 
The fitness level method is a popular tool  for analyzing the hitting time of elitist evolutionary algorithms. Its idea is to divide the search space into multiple fitness levels and estimate lower and upper bounds on the hitting time using   transition probabilities between fitness levels. However, the lower bound generated by this method is often loose. An open question regarding   the fitness level method is what are the tightest lower and upper time bounds that can be constructed based on  transition probabilities between fitness levels. To answer this question, { we  combine  drift analysis with fitness levels and define the tightest bound problem as a constrained multi-objective optimization problem subject to fitness levels.}  The tightest metric bounds from fitness levels are constructed and proven for the first time.  Then linear bounds are derived from metric bounds and a framework is established that can be used to develop different fitness level methods for different types of linear bounds.  The framework is generic and promising, as it can be used to draw  tight time bounds  on both fitness landscapes without and with shortcuts. This is demonstrated in the example of the (1+1) EA maximizing the TwoMax1 function. 
\end{abstract}

\section{Introduction}
\subsection{Background}
The time complexity of evolutionary algorithms (EAs) is an  important  topic in the EA theory~\citep{oliveto2007time, yu2008new,doerr2017time,huang2019experimental}. The computation time of EAs  can be measured by either the number of  generations   to find an optimum for the first time, called hitting time \citep{he2001drift}, or the number of  fitness evaluations, called running time~\citep{he2017average}.   The analysis of running time is more complicated as it is related to the population size~\citep{he2002individual,chen2009new,he2017average}, and the population size often varies from generation to generation. Therefore, we will limit this paper's discussion to hitting time.  Several methods have been proposed for analyzing hitting time of EAs,  such as drift analysis~\citep{he2001drift}, Markov chains~\citep{he2002individual,he2003towards} and fitness level partition~\citep{wegener2001theoretical}.    Each method has its own advantages and disadvantages. Drift analysis is a powerful tool in which an appropriate distance is constructed as the bound on hitting time~\citep{he2001drift,oliveto2011simplified,doerr2012multiplicative}.  According to the theory of absorbing Markov chains, the  exact  hitting time of EAs can be calculated from the fundamental matrix of absorbing Markov chains~\citep{he2003towards}. But   the calculation of the fundamental matrix is too complex for most  EAs. Therefore, hitting time is estimated by replacing the original chain with a slower or faster chain~\citep{he2003towards,zhou2009comparative}.

The fitness level method~\citep{wegener2001theoretical,wegener2003methods} is a popular   tool used to estimate hitting time of elitist EAs~\citep{antipov2018runtime,corus2020hypermutations,rajabi2020self,quinzan2021evolutionary,aboutaib2022runtime,malalanirainy2022runtime,oliveto2022tight}. 
The basic concept of this method is to divide the search space into multiple  ranks  $(S_0, \cdots, S_K)$, called fitness levels, based on the fitness value from high to low, where the highest rank $S_0$ is the optimal set; then calculate transition probabilities between fitness levels, that is, $p(X_k, S_\ell)$ from $X_k \in S_k$ to $S_\ell$ (where $1\le \ell<k\le K$); 
finally,  estimate a  bound $d_k$ on the hitting time of the EA starting from a level $S_k$.  The method was combined with other techniques such as tail bounds~\citep{witt2014fitness} and stochastic domination~\citep{doerr2019analyzing}. The fitness level method is available for elitist EAs. Although the  level partition is also used to analyze non-elitist EAs~\citep{corus2017level,case2020self}, they should be considered as a different method.

In this paper, we express   time bounds   from fitness levels in linear forms as follows:
\begin{align} 
\label{equLinearLower} 
 \mbox{lower bound }&  d_k  = \sum^{k}_{\ell=1}\frac{c_{k,\ell}}{\max_{X_\ell \in S_\ell}  
p(X_\ell, S_0 \cup \cdots \cup S_{\ell-1})},   \\ 
\label{equLinearUpper} 
\mbox{upper bound } & d_k = \sum^{k}_{\ell=1}\frac{c_{k,\ell}}{\min_{X_\ell \in S_\ell}  
p(X_\ell, S_0 \cup \cdots \cup S_{\ell-1})}, 
\end{align} 
where  $c_{k,\ell}$ are coefficients and $p(X_\ell, S_0 \cup \cdots \cup S_{\ell-1})$ represents the transition probability from a state $X_\ell \in S_\ell$ to the union of levels $S_0 \cup \cdots \cup S_{\ell-1}$.

{How to calculate coefficients $c_{k,\ell}$ for tight bounds is the key topic in the fitness level method.}
\cite{wegener2003methods} assigned  $c_{k,k}=1$, other coefficients $c_{k,\ell}=0$ for the lower bound and $c_{k,\ell}=1$ for the upper bound where $k > \ell$. This assignment is good at obtaining a tight upper bound, but not good at obtaining a tight lower bound.  
Several efforts have been made to improve the lower bound since then. 
\cite{sudholt2012new} made an improvement using a constant coefficient $c_{k,\ell}= c$ (called viscosity) for $k > \ell$ and $c_{k,k}=1$, and gave   tight lower time bounds of the (1+1) EA on several unimodal functions such as LeadingOnes, OneMax, long $k$-paths. Recently, 
\cite{doerr2021lower} made another improvement using  coefficients  $c_{k,\ell}=c_{\ell}$  (called visit probability) and provided tight lower bounds of the  (1+1) EA on LeadingOnes,  OneMax, long $k$-paths jump function. 
{However, in this paper we show that  the lower bounds based on the above coefficients $c$ or $c_\ell$ are loose on  landscapes with shortcuts. A shortcut  means that an EA may skip  some intermediate fitness levels with a large probability.  Therefore, it is necessary to improve the lower bound  of EAs on fitness landscapes with shortcuts.}

\subsection{New research and main results in this paper}
The aim of this paper is to explore two research questions that have not been addressed before.
{
\begin{enumerate}
    \item What are the tightest lower and upper bounds based on  transition probabilities between fitness levels? 
    \item Is it possible to use fitness level methods to draw tight lower bounds on fitness landscapes with shortcuts?
\end{enumerate}
}

To answer the questions, drift analysis with fitness levels is developed for constructing lower and upper bounds on the  hitting time of elitist EAs. The fitness level method is viewed as  a combination of drift analysis and fitness levels. Given a fitness level partition $(S_0, \cdots, S_K)$, a distance $d_k$ between $S_k$ and $S_0$ is assigned to each fitness level $S_k$, where 
$d_k$ is constructed using transition probabilities between fitness levels. Since $d_k$ is related to distance,  it is called a metric bound.
Then  by drift analysis, it is proved that $d_k$ is  a lower or upper bound   on the hitting time   of the EA starting from $S_k$, and the best $d^*_k$ is the tightest metric bound.  
 The new contributions and  results are summarized in three parts.
\begin{enumerate}
\item First, we  construct metric bounds from fitness levels  and prove that the best metric bounds are the tightest. The metric lower  bound  from fitness levels is expressed recursively as
\begin{align}
\label{equMetricLower}
    d_k \le \min_{X_k \in S_k} \left\{\frac{1}{p(X_k, S_0 \cup \cdots \cup S_{k-1})} + \sum^{k-1}_{\ell=1} \frac{p(X_k,S_{\ell})}{p(X_\ell, S_0 \cup \cdots \cup S_{\ell-1})} d_{\ell}\right\}. 
\end{align}    Similarly, the  metric upper  bound  from fitness levels   is expressed recursively as
\begin{align}
\label{equMetricUpper}d_k \ge \max_{X_k \in S_k} \left\{\frac{1}{p(X_k, S_0 \cup \cdots \cup S_{k-1})} + \sum^{k-1}_{\ell=1} \frac{p(X_k,S_{\ell})}{p(X_\ell, S_0 \cup \cdots \cup S_{\ell-1})} d_{\ell}\right\}.
\end{align}

The tightest lower  or upper  bound is reached when Inequality   \eqref{equMetricLower} or \eqref{equMetricUpper}  becomes an equality.

\item  Secondly, we construct general linear bounds  from metric bounds (\ref{equMetricLower}) and (\ref{equMetricUpper}).  Coefficients in  the   linear lower  bound (\ref{equLinearLower}) satisfy $c_{k,k}=1$ and the following linear inequalities:
\begin{align} 
\label{equCoefLower}    
c_{k,\ell}  
        \le \min_{X_k \in S_k}    \frac{p(X_k,S_{\ell})+ \sum^{k-1}_{j=\ell+1} p(X_k,S_{j}) c_{j, \ell}}{p(X_k, S_0 \cup \cdots \cup S_{k-1})}, \qquad 0<\ell <k.
\end{align}   

Coefficients in the linear upper  bound (\ref{equLinearUpper}) satisfy $c_{k,k}=1$ and the following linear inequalities:
\begin{align} 
\label{equCoefUpper}   
c_{k,\ell}  
        \ge \max_{X_k \in S_k}    \frac{p(X_k,S_{\ell})+ \sum^{k-1}_{j=\ell+1} p(X_k,S_{j}) c_{j, \ell}}{p(X_k, S_0 \cup \cdots \cup S_{k-1})}, \qquad 0<\ell <k.
\end{align} 

Previous bounds \citep{wegener2003methods,sudholt2012new,doerr2021lower} can be regarded as special cases of $c_{k,\ell}= 0,1, c,  c_\ell$. For the sake of discussion, the family of linear bounds   \eqref{equCoefLower}  and \eqref{equCoefUpper}  are named \emph{Type-$c_{k,\ell}$ bounds}. Similarly, \emph{Type-$0,1$, $c$ and $c_\ell$ bounds} stand for linear bounds with coefficients $c_{k,\ell} =0,1$, $c$ and $c_\ell$, respectively.

\item Finally, we demonstrate the advantage of the Type-$c_{k, \ell}$ lower bound over Type-$c$ and $c_\ell$ lower bounds. For the (1+1) EA maximizing the TwoMax1 function, we prove that our Type-$c_{k, \ell}$ lower bound is $\Omega(n \ln n)$, but Type-$c$ and $c_\ell$ lower bounds are only $O(1)$.
\end{enumerate}  

The paper is organized as follows:  Section~\ref{secEAs} provides the foundation of theoretical analysis. Section~\ref{secReview} reviews  existing fitness level methods and explains the necessity of improving previous lower linear bounds.  Section~\ref{secDrift} proposes drift analysis with fitness levels, constructs new metric bounds and proves  they are the tightest. Section~\ref{secLinear} constructs general linear  bounds  and  presents different explicit expressions of coefficients. Section~\ref{secCase} shows the application of general linear bounds.   
Section~\ref{secConclusions} concludes the work.


\section{Preliminary}
\label{secEAs} 

\subsection{Elitist EAs and Markov chains}
A maximization problem is considered in the paper:  
$f_{\max}=\max f( {x})$
where $f(x)$ is defined on a finite set. In EAs,   an individual $x$ represents a solution. A population  consists of several  individuals, denoted by   $X$. The fitness of a population  $f(X)=\max\{ f(x); x \in X \}.$  
Let $S$ denote the set of all populations and $S_{\mathrm{opt}}$  the set of optimal populations $X_{\mathrm{opt}}$ such that $f(X_{\mathrm{opt}}) =f_{\max}$.  This paper studies elitist EAs that maximize $f(x)$. 
Let $X^{[t]}$ denote the $t$-th generation population. 

\begin{definition}{An EA is called \emph{elitist} if     $ f(X^{[t]}) \ge  f(X^{[t-1]})$.}
\end{definition} 
A simple elitist EA is the (1+1) EA  using bitwise mutation and elitist selection for maximizing a pseudo-Boolean function:  $f(x)$  where $x=(x_1, \cdots, x_n) \in \{0,1\}^n$. The (1+1) EA does not use a population, but keeps only an individual. 

\begin{algorithm}[ht]
\caption{{The (1+1) EA that maximizes a pseudo-Boolean function $f(x)$}}
\label{alg2}
\begin{algorithmic}[1] 
\STATE    initialize a  solution  $x$ and let $ {x}^{[0]}=x$;
\FOR{$t=1,2,\cdots$}
\STATE   flip  each bit of $x$ independently with  probability $\frac{1}{n}$ and generate a solution $y$;
\STATE  if $f(y)\ge f(x)$, then  let $ {x}^{[t]}= y$, otherwise $ {x}^{[t]}=x$.
\ENDFOR
\end{algorithmic}
\end{algorithm}

We assume that  EAs  are modeled by homogeneous Markov chains~\citep{he2003towards,he2016average}. The set $S$ is the state space of a  Markov chain and a population $X$ is a state.  The Markov chain property means that the next state $X^{[t+1]}$ depends only on the current state $X^{[t]}$, that is, $\Pr(X^{[t+1]}   \mid X^{[t]}, \dots, X^{[0]} )=\Pr(X^{[t+1]} \mid X^{[t]})$.  
The  homogeneous property means that the transition probability  from a state $X$ to another state $Y$  does not change over the generation $t$,   that is for any $t$, $\Pr(X^{[t+1]} =Y \mid X^{[t]} =X)=p(X,Y)$.

\subsection{Hitting time and drift analysis}

Hitting time is the first time when an EA finds  an optimal solution.   

\begin{definition}
Given an elitist EA for maximizing $f(x)$, assume that the initial population $X^{[0]}=X$. \emph{Hitting time} $\tau(X)$ is  the number of generations when an optimum is found for the first time, that is, $
    \tau(X) = \min \{ t\ge 0, X^{[t]} \in S_{\mathrm{opt}} \mid X^{[0]}=X\}.$  The \emph{mean hitting time} $m(X)$ is  the expected value of $\tau(X)$, that is $
    m(X) =\mathrm{E}[\tau(X)].$
Assume that   $X^{[0]}$ is chosen at random, the \emph{mean hitting time} is the expected value
$
    m(X^{[0]}) =\sum_{X \in S} m(X)\Pr(X^{[0]}=X).$
\end{definition}

Drift analysis was introduced by 
\cite{he2001drift} to the analysis of hitting time of EAs. It is based on the intuitive idea: time = distance/drift.    A non-negative function $d(X)$    measures the distance between  $X$ and
the optimal set. By default, let $d(X)=0$ if  $X \in S_{\mathrm{opt}}$. A distance function $d(X)$ is called a lower time bound if for all $X$, $d(X) \le m(X)$, while  $d(X)$ is called an upper time bound if for all $X$, $d(X) \ge m(X)$. 

There are several variants of drift analysis~~\citep{he2001drift,oliveto2011simplified,doerr2012multiplicative,doerr2013adaptive}. For a complete review of drift analysis, see \citep{kotzing2019first,lengler2020drift}. {In this paper, we establish drift analysis with fitness levels based on the Markov chain version of drift analysis~\citep{he2003towards}. For elitist EA  that cannot be modeled by a Markov chain,  it is still possible to establish drift analysis with fitness levels by the super-martingale version of   drift analysis~\citep{he2001drift}}. 
\begin{definition}
The \emph{drift} $\Delta d(X)$ is the  distance change  per generation  at the state $X$,  
\begin{align}
\label{equDrift}
      \Delta d(X)=  d(X) - \sum_{Y \in S} p(X,Y) d(Y).   
\end{align}
\end{definition}

\begin{lemma} \cite[Theorem 3]{he2003towards} 
\label{lemmaLowerBound1}
If for any $X\notin S_{\mathrm{opt}}$, the drift  $\Delta d(X) \le 1$, then the mean hitting time $m(X) \ge  d(X)$.
\end{lemma}

\begin{lemma} \cite[Theorem 2]{he2003towards}  
\label{lemmaUpperBound1}
If for any $X\notin S_{\mathrm{opt}}$, the drift  $\Delta d(X) \ge 1$, then the mean hitting time $m(X) \le  d(X)$.
\end{lemma}

We use the asymptotic notation~\citep{knuth1976big} to describe the tightness of lower and upper bounds. The worst-case time complexity of an EA is measured by the maximum value of the mean hitting time $\max_{X \in S} m(X) $.  We say that a lower bound $d(X)$ is tight if  $\max_{X \in S} d(X)=\Omega(\max_{X \in S} m(X))$, and an upper bound $d(X)$  is tight if $\max_{X \in S} d(X)=O(\max_{X \in S} m(X))$. 
{ We also divide coefficients $c_{k,\ell}$ into two categories: large coefficients $c_{k,\ell}$ if $c_{k,\ell}=\Omega(1)$ and small coefficients $c_{k,\ell}$ if $c_{k,\ell}=o(1)$.  }

\subsection{Fitness level partition and transition probabilities between fitness levels}
\label{subFitnessLevels1}
The fitness level method depends on a fitness level partition.
\begin{definition} {In a \emph{fitness level partition},} the state space $S$ is divided into $K+ 1$ disjoint subsets (ranks) $(S_0, \cdots, S_K)$ according to the fitness from high to low     such that  (i) the highest rank $S_0=S_{\mathrm{opt}}$, (ii)  for all $X_k \in S_k$ and $X_{k+1} \in S_{k+1}$, the rank order holds: $f(X_k) > f(X_{k+1})$.  
Each rank  is called a \emph{fitness level}. 
\end{definition}  

The fitness level method is based on transition probabilities between fitness levels. We assume that the following transition probabilities are known in advance.  The notation $p(X_k,  S_{\ell})$ denotes the transition probability from  a state $X_k \in S_k$ to the level  $S_{\ell}$.   
\begin{equation} 
\label{equProbability-p}  
         p(X_k,  S_{\ell})= \Pr(X^{[t+1]} \in S_{\ell} \mid   X^{[t]} =X_k) . 
\end{equation} 
Its minimal and maximal values are denoted as follows: 
\begin{equation*}
\begin{split} 
 p_{{\scriptscriptstyle\min}}(X_{k},S_{[0,k-1]})= \min_{X_k \in S_k} p(X_k,S_{[0,k-1]}) \mbox{ and }   p_{{\scriptscriptstyle\max}}(X_{k},S_{[0,k-1]}) = \max_{X_k \in S_k} p(X_k,S_{[0,k-1]}) . 
\end{split}
\end{equation*}  

Other transition probabilities between levels are derived from $p(X_k,  S_{\ell})$.   
Let $[i,j]$ denote the integer set $\{i, i+1, \cdots, j-1,j\}$ and  $S_{[i,j]}$ denote the union of levels $S_i \cup S_{i+1} \cup \cdots\cup S_{j-1}\cup S_j$.  The transition probability from a state $X_k\in S_k$ to the union $S_{[i,j]}$ is denoted by $p(X_{k},S_{[i,j]})$, that is, $p(X_{k},S_{[i,j]})=\sum^j_{\ell=i}p(X_k,  S_{\ell})$.  

The notation $r(X_k, S_{\ell})$ denotes the conditional probability  
\begin{equation}  
\label{equProbability-r}
    r(X_k, S_{\ell}) = \frac{p(X_k,S_{\ell})}{p(X_k, S_{[0,k-1]})} .  
\end{equation}  
Its minimal and maximal values are denoted as follows: 
\begin{equation*}  
\begin{split}  
 &r_{{\scriptscriptstyle\min}}(X_k,  S_{\ell})= \min_{X_k \in S_k} p(X_k,  S_{\ell}) \mbox{ and }     r_{{\scriptscriptstyle\max}}(X_k,  S_{\ell})= \max_{X_k \in S_k} r(X_k,  S_{\ell}). 
\end{split}
\end{equation*}

Table~\ref{tab:notation} lists main symbols used in this paper.

\begin{table}[ht]
    \centering 
    \caption{Notation used in the paper.}
    \label{tab:notation} 
\begin{tabular}{|l|l|}
\hline 
$S_k$ &  a fitness level 
\\\hline
$S_{[i,j]}$ &  the union of fitness levels $S_i \cup S_{i+1} \cdots \cup S_{j-1}\cup S_j$ where $i<j$  
\\\hline
$X_k$ & a state in  $S_k$
\\\hline
$m(X_k)$ & the mean hitting time when the EA starts from   $X_k$ 
\\\hline
$p(X_k,S_{\ell})$  & the transition probability from  $X_k$ to  $S_{\ell}$
\\\hline
$p(X_k,S_{[i,j]})$  & the transition probability from  $X_k$ to  $S_{[i,j]}$
\\\hline
    $r(X_k, S_{\ell})$ & the conditional probability $  \frac{p(X_k,S_{\ell})}{p(X_k, S_{[0,k-1]})}$  
\\\hline
 $c_{k,\ell}, c_\ell, c $   & coefficients in linear bounds
\\\hline
\end{tabular}
\end{table}

\subsection{Shortcuts}
\label{secShortcuts}  Intuitively, the behavior of an elitist EA searching for the maximum value of a fitness function can be viewed as climbing on a fitness landscape. For most fitness landscapes, an EA can take different paths from lower to higher fitness levels, some of which are shorter than others. A shortcut implies that an intermediate fitness level is skipped. {In this paper, we provide a formal definition of shortcuts as follows. }  

\begin{definition}
   Given an elitist EA for maximizing a function $f(x)$ and a fitness level partition $(S_0, \cdots, S_K)$, there exists a \emph{shortcut}  on the fitness landscape  if for some $1\le \ell,k \le K$ and $X_k \in S_k$,    the conditional  probability  
   \begin{align}
    \label{equShortcut}
       \frac{p(X_k, S_{\ell})}{p(X_k, S_{[0,\ell]})}  =o(1).
    \end{align}   
\end{definition} 
According to \eqref{equShortcut}, the conditional   probability  of the EA starting from $X_k$ to visit $S_\ell$ is $o(1)$, so the conditional   probability of the EA starting from $X_k$ to visit $S_{[0,\ell-1]}$ is $1-o(1)$. Thus, the EA  skips the level $S_\ell$ with a large conditional probability $1-o(1)$.  

Fitness landscapes can be divided into two categories: with shortcuts and without shortcuts.
Let us demonstrate two examples.
The first example is  the (1+1) EA that maximizes the OneMax function: 
\[\mathrm{OM}(x)= |x|,\quad x =(x_1, \cdots, x_n)\in \{0,1\}^n,\] 
where $|x|=x_1+\cdots +x_n$. The state space can be divided into $n+1$ levels $(S_0, \cdots, S_n)$, where $S_k =\{x \in \{0,1\}^n; \mathrm{OM}(x) =n-k\}$. Figure~\ref{fig:shortcuts} shows that no shortcut exists on the fitness landscape of the (1+1) EA on OneMax.

\begin{figure}[ht]
    \centering
\includegraphics[width=0.495\textwidth]{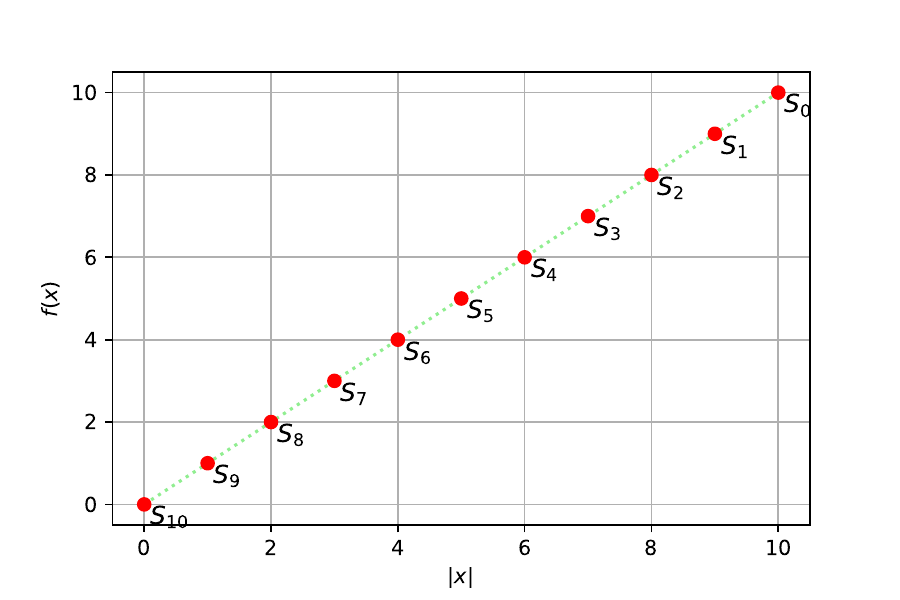} \includegraphics[width=0.495\textwidth]{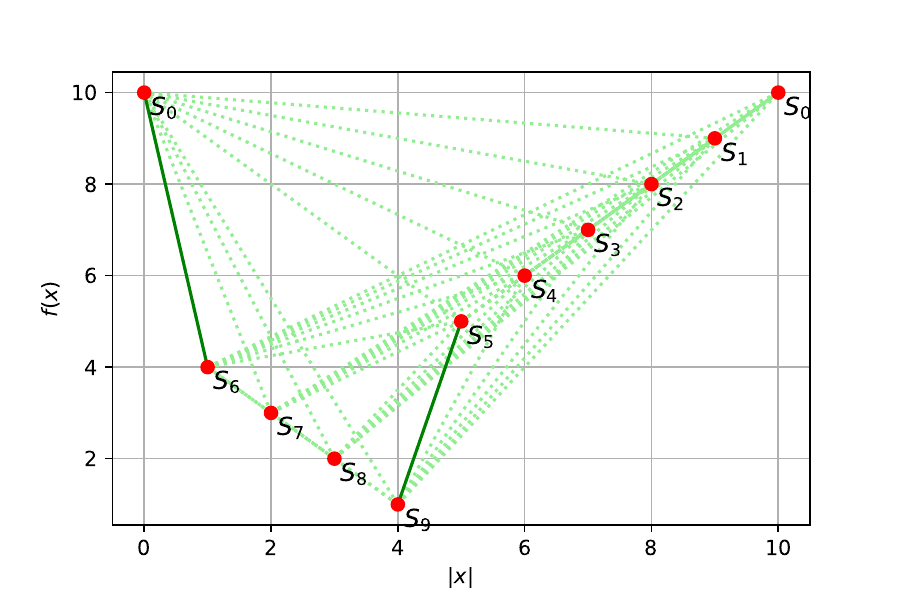}
    \caption{Left: the (1+1) EA on OneMax$(x)$  where $n=10$. Right: The (1+1) EA on TwoMax1$(x)$     where $n=10$. Dotted lines represent transitions. Solid  lines are two shortcuts: $S_6 \to S_0$ skipping $S_1, \cdots, S_5$, and $S_9\to S_5$ skipping $S_6$.  }
    \label{fig:shortcuts}
\end{figure}

The second example is  the (1+1) EA  maximizing the TwoMax1 function.     \begin{align*}
        \mathrm{TM1}(x) = \left\{\begin{array}{lll}
             n& \mbox{if } |x| =0  \mbox{ or } |x|=n, \\
            |x| & \mbox{if } |x| \ge  \frac{n}{2},\\
            \frac{n}{2}-|x|& \mbox{else},
        \end{array}\right. 
    \end{align*} 
where  $n$ is a large even integer. 
There are two maxima at $|x|=0 $ and $ n$. TwoMax1 is a variant of the TwoMax function defined in~\citep{he2015easiest}. The search space can be split into $n$ fitness levels $(S_0, \cdots, S_{n-1})$ from high to low: $S_k  = \{x \in \{0,1\}^n: \mathrm{TM1}(x) =n-k\} $.   Figure~\ref{fig:shortcuts}  displays two shortcuts on the fitness landscape of the (1+1) EA on TwoMax1. The two solid lines represent  shortcuts.  The first shortcut is $S_{n/2+1}\to S_0$ skipping  fitness levels $S_1, \cdots, S_{n/2}$. The second shortcut is  $S_{n-1}   \to  S_{n/2}$ skipping   $S_{n/2+1}$. We omit the rigorous proof of these shortcuts. 
  

\section{Review and discussion of existing fitness level methods}
\label{secReview} 
\subsection{Existing fitness level methods}
\label{subFitnessLevels2}  
Given a fitness level partition $(S_0, \cdots S_K)$, previous results are summarized as follows. 
\cite{wegener2003methods}  gave the simple Type-$0$  lower bound  and Type-$1$ upper time bound  
as shown in Propositions~\ref{proposition1} to~\ref{proposition2}.   

\begin{proposition} \citep[Lemma 1]{wegener2003methods}
\label{proposition1} For all $k\ge 1$ and $X_k \in S_k$, the hitting time 
$      m(X_k) \ge  \frac{1}{p_{{\scriptscriptstyle\max}}(X_{k},S_{[0,k-1]})}.$ 
\end{proposition}

\begin{proposition}  
\label{proposition2}  \citep[Lemma 2]{wegener2003methods} For all $k\ge 1$ and $X_k \in S_k$,
the hitting time 
$m(X_k) \le \sum^{k}_{\ell=1} \frac{1}{p_{{\scriptscriptstyle\min}}(X_{\ell},S_{[0,\ell-1]})}.$
\end{proposition}

\cite{sudholt2012new} improved the lower bound using a constant coefficient $c$ (called  viscosity). 
\begin{proposition} \citep[Theorem 3]{sudholt2012new}
\label{proposition3} For any $0\le \ell <k \le K$, 
let    $p(X_k, S_{\ell}) \le p_{{\scriptscriptstyle\max}}(X_{k},S_{[0,k-1]}) \, r_{k,\ell}$ and $\sum^{k-1}_{\ell=0} r_{k,\ell}=1$.  Assume that there is some $0\le c\le 1$ such that for any $1\le l < k\le K$, it holds
${r_{k,\ell}} \ge c {\sum^{\ell}_{j=0} r_{k,j}}$. Then the mean hitting time     \begin{align*}
       m(X^{[0]}) \ge \sum^K_{k=1} \Pr(X^{[0]} \in S_k)  \left(\frac{1}{p_{{\scriptscriptstyle\max}}(X_{k},S_{[0,k-1]})} +\sum^{k-1}_{\ell= 1} \frac{c}{p_{{\scriptscriptstyle\max}}(X_{\ell},S_{[0,\ell-1]})}\right). 
    \end{align*} 
\end{proposition}

An issue with Proposition~\ref{proposition3} is that the above Type-$c$ lower bound is loose on fitness landscapes with shortcuts. { This claim is demonstrated with the example in Section ~\ref{secTwoMax1-c}.}

{\cite{sudholt2012new} also gave an upper bound similar to Proposition~\ref{proposition3} but with an extra  condition    $(1 - c)p_{{\scriptscriptstyle\min}}(X_{\ell+1},S_{[0,\ell]}) \le  p_{{\scriptscriptstyle\min}}(X_{\ell},S_{[0,\ell-1]})$.}

\begin{proposition}\citep[Theorem 4]{sudholt2012new}
\label{proposition4} For any $0\le \ell <k \le K$, 
let  $p(X_k, S_{\ell}) \ge p_{{\scriptscriptstyle\min}}(X_{k},S_{[0,k-1]}) \, r_{k,\ell}$ and  $\sum^{k-1}_{\ell=0} r_{k,\ell}=1$. Assume that there is some $0\le c\le 1$ such that for all $1\le l < k \le K$, it holds $
        r_{k,\ell} \le c \sum^{\ell}_{j=0} r_{k,j}$. 
        Further, assume that  for all $1 \le l \le K - 2$, it holds $(1 - c)p_{{\scriptscriptstyle\min}}(X_{\ell+1},S_{[0,\ell]}) \le  p_{{\scriptscriptstyle\min}}(X_{\ell},S_{[0,\ell-1]})$. Then the mean hitting time  
    \begin{align*}
       m(X^{[0]})  \le  \sum^K_{k=1} \Pr(x^{[0]} \in S_k) \left( \frac{1}{p_{{\scriptscriptstyle\min}}(X_{k},S_{[0,k-1]})} +\sum^{k-1}_{\ell= 1} \frac{c}{p_{{\scriptscriptstyle\min}}(X_{\ell},S_{[0,\ell-1]})} \right).
    \end{align*}
\end{proposition}

\cite{doerr2021lower}  further improved the lower bound  
 using a   coefficient $c_\ell$ (called visit probability).   

\begin{proposition} \label{proposition5}  \citep[Theorem 8]{doerr2021lower} 
  For all $\ell =1, \cdots, K$, let $c_{\ell}$ be a lower bound on the probability of there being a $t$ such that $X^{[t]} \in S_{\ell}$. Then the mean hitting time  
$
   m(X^{[0]})  \ge \sum^{K}_{\ell=1} \frac{c_{\ell}}{p_{{\scriptscriptstyle\max}}(X_{\ell},S_{[0,\ell-1]})}.$
\end{proposition}  

Proposition~\ref{proposition5} does not provide an explicit  formula to calculate the  visit probability $c_\ell$. Therefore, \cite{doerr2021lower} proposed a  method of calculating $c_{\ell}$ as follows.  
\begin{lemma} \citep[Lemma 10]{doerr2021lower}
\label{lemmaVisitProbability}
For $1\le \ell\le K$, suppose there is $c_{\ell}$ such that, for all $X \in S_{[\ell+1, K]}$ with $p(X, S_{[0,\ell]}) > 0$,  
\begin{align}
    &c_\ell \le   \frac{p(X, S_{\ell})}{p(X, S_{[0,\ell]})},   
  \label{equDKCondition1}\\ 
\mbox{and }& c_\ell \le \frac{\Pr(X^{[0]} \in S_{\ell})}{\Pr(X^{[0]} \in S_{[0,\ell]})}. \label{equDKCondition2}
\end{align} 
Then $c_{\ell}$ is  a lower bound for visiting $S_{\ell}$ as required by Proposition~\ref{proposition5}.
\end{lemma}

{
An issue with Lemma~\ref{lemmaVisitProbability} is that the above Type-$c_\ell$ lower bound    \eqref{equDKCondition1}  is loose on fitness landscapes with shortcuts. An example  in Section~\ref{secTwoMax1-cl} shows this issue.   }

\cite{doerr2021lower} also gave an upper bound similar to Proposition~\ref{proposition5}.  
But the proposition does not provide an explicit  formula to calculate the visit probability $c_\ell$ using transition probabilities between fitness levels.

\begin{proposition}  \label{proposition6}\citep[Theorem 9]{doerr2021lower}
 For all $\ell=1, \cdots, K$, let $c_{\ell}$ be an upper bound on the probability of there being a $t$ such that $X^{[t]} \in S_{\ell}$. Then the mean hitting time  
    $
    m(X^{[0]})  \le \sum^{K}_{\ell=1} \frac{c_{\ell}}{p_{{\scriptscriptstyle\min}}(X_{\ell},S_{[0,\ell-1]})}.$ 
\end{proposition}

\subsection{Case Study 1: A loose Type-$c$ lower bound for the (1+1) EA on TwoMax1}
\label{secTwoMax1-c}
In this case study, we show that the Type-$c$ lower bound  by Proposition~\ref{proposition3}  is loose  on fitness landscapes with shortcuts.
Consider  the (1+1) EA  maximizing TwoMax1. 
Assume that the EA starts from $S_{n-1}$. We  aim to
 prove the lower bound    by Proposition~\ref{proposition3} is only $O(1)$, that is
\begin{align*} 
       d_{n-1}  = \frac{1}{p ({x}_{k},S_{[0,k-1]})} +\sum^{k-1}_{\ell= 1} \frac{c}{p ({x}_{\ell},S_{[0,\ell-1]})}=O(1). 
    \end{align*} 

The transition probability \underline{${p({x}_\ell,S_{[0,\ell-1]})}$ (where $\ell =1, \cdots,  n/2$)}  is calculated as follows. Since a state $x_\ell \in S_\ell$  has $\ell$ zero-valued bits.  The transition   from $x_\ell$ to $S_{\ell-1} \subset S_{[0,\ell-1]}$ happens if 1 zero-valued bit   is flipped and other bits are unchanged. Thus
\begin{align}
\label{equ:wk1}
p (x_\ell, S_{[0,\ell-1]}) \ge \binom{\ell}{1} \frac{1}{n} \left(1-\frac{1}{n}\right)^{n-1} \ge \frac{\ell}{n} e^{-1}.
\end{align}  

The transition probability  \underline{${p({x}_\ell,S_{[0,\ell-1]})}$  (where $\ell =n/2+1, \cdots,  n-1$)} is calculated as follows. Since a state  $x_\ell \in S_\ell$ has $\ell-n/2$ one-valued bits, the transition   from $x_\ell$ to $S_{[0,\ell-1]}$ happens if 1 one-valued bit is flipped and other bits are unchanged. Thus
 \begin{align}
\label{equ:wk2}
p (x_\ell, S_{[0,\ell-1]})   \ge \binom{\ell-n/2}{1} \frac{1}{n} \left(1-\frac{1}{n}\right)^{n-1} \ge \frac{\ell-n/2}{n  } e^{-1}.
\end{align} 

By  (\ref{equ:wk1}) and (\ref{equ:wk2}), we get a lower bound
\begin{align} 
\label{equ:clowerbound} 
d_{n-1}   \le \frac{en}{(n-1)-\frac{n}{2}}+  c \left(\sum^{n/2}_{\ell=1} \frac{en}{\ell} +  \sum^{n-2}_{\ell=n/2+1} \frac{en}{\ell-n/2}\right) = O(1).
\end{align} 

The \underline{coefficient $c$}  is calculated using the shortcut $S_{n/2} \to S_0$ as follows. We replace $r_{k,\ell} $ in Proposition~\ref{proposition3} with $r(x_k, S_\ell)$.
{Let $k=n/2 $ and $ \ell=1$.}
 \begin{align*}  
    c \le \frac{r_{n/2+1,1}}{r_{n/2+1,0} + r_{n/2+1,1}}=\frac{p(x_{n/2+1},S_1)}{p(x_{n/2+1},S_0) + p(x_{n/2+1},S_1)}.  
\end{align*} 
Since
$x_{n/2+1} \in S_{n/2+1}$ has $n-1$ zero-valued bits and $1$ one-valued bit, 
the transition from $x_{n/2+1}$  to $S_0$ happens if the one-valued bit   is flipped and other bits are unchanged.  
\begin{align} 
\label{equc1}
 p(x_{n/2+1},S_0) \ge   \frac{1}{n} \left(1-\frac{1}{n}\right)^{n-1}.  
\end{align}

Since a state in $S_1$ has 1 zero-valued bit and $n-1$ one-valued bits, the transition from $x_{n/2+1}$  to $S_1$  happens if and only if $n-2$ zero-valued bits  are flipped and other bits are unchanged. 
\begin{align}
\label{equc2}
 p(x_{n/2+1},S_1)= \left(\frac{1}{n}\right)^{n-2} \left(1-\frac{1}{n}\right)^{2}.
\end{align}
By (\ref{equc1}) and (\ref{equc2}), we have
\begin{align}
\label{equ:case1-coeff}
    {c \le  O(n^{-n+3})}.  
\end{align} 

Inserting \eqref{equ:case1-coeff} into (\ref{equ:clowerbound}), we get a lower bound
 \begin{align*} 
d_{n-1}   \le \frac{en}{(n-1)-\frac{n}{2}}+  O(n^{-n+3}) \left(\sum^{n/2}_{\ell=1} \frac{en}{\ell} +  \sum^{n-2}_{\ell=n/2+1} \frac{en}{\ell-n/2}\right) = O(1).
\end{align*}
The lower bound is $O(1)$, much looser than the actual hitting time $\Theta(n \ln n)$. 

\subsection{Case Study 2:  A loose Type-$c_\ell$ lower bound for the (1+1) EA on TwoMax1} \label{secTwoMax1-cl}

In this case study, we show that the Type-$c_\ell$ lower bound  by Lemma~\ref{lemmaVisitProbability} and Proposition~\ref{proposition5}   is  loose on fitness landscapes with shortcuts. {Consider  the (1+1) EA  maximizing TwoMax1 under  random initialization. This means $\Pr(x^{[0]} \in S_{n-1})>0$.} We 
 prove that the lower bound     by Lemma~\ref{lemmaVisitProbability} and Proposition~\ref{proposition5}  is only $O(1)$, that is 
 $$
  d  =\sum^{n-1}_{\ell=1} \frac{c_{\ell}}{p_{{\scriptscriptstyle\max}}({x}_{\ell},S_{[0,\ell-1]})}=O(1).$$

By (\ref{equ:wk1}) and (\ref{equ:wk2}), according to Proposition~\ref{proposition5}, we get a lower time bound  
\begin{align} 
\label{equ:Type-cl-d}
    d    \le \frac{en}{(n-1)-\frac{n}{2}}  + \sum^{n/2}_{\ell=1}  c_{\ell} \frac{en}{ \ell} +\sum^{n-2}_{\ell=n/2+1} c_{\ell}\frac{en}{\ell-n/2}.  
\end{align}

Coefficients   
are calculated by Condition (\ref{equDKCondition1}) but without  Condition \eqref{equDKCondition2}.  
{
Since our target is an upper bound on $c_\ell$, Condition \eqref{equDKCondition2} doesn't matter. It only reduces the $c_\ell$ value.}

\underline{For $\ell=1, \cdots, n/2$,  
the coefficient $c_\ell$}  is calculated using the shortcut $S_{n/2+1} \to S_0$ as follows. According to   Condition (\ref{equDKCondition1}),  for $\ell=1, \cdots, n/2$ and $x_{n/2+1} \in S_{n/2+1}$, 
\begin{equation*}  
c_\ell \le   \frac{p(x_{n/2+1}, S_{\ell})}{p(x_{n/2+1}, S_{[0,\ell]})}.
\end{equation*}
Since $x_{n/2+1}$ has 1 one-valued bit and a state in $S_{\ell}$ has $n-\ell$ one-valued bits, the transition from $x_{n/2+1}$ to $S_{\ell}$ happens only if $n-1-\ell$ zero-valued bits are flipped. Thus,
\begin{equation}
\label{equ:prob1a}
p(x_{n/2+1}, S_\ell) \le \binom{n-1}{n-1-\ell} \left(\frac{1}{n}\right)^{n-1-\ell}.
\end{equation}

The transition from $x_{n/2+1}$ to $S_0 \subset S_{[0,\ell]}$ happens  if the one-valued bit is flipped and other bits are unchanged. Thus,
\begin{equation}
\label{equ:prob1b}
p(x_{n/2+1}, S_{[0,\ell]}) \ge \frac{1}{n} \left(1-\frac{1}{n}\right)^{n-1} \ge \frac{1}{en}.
\end{equation}

Combining (\ref{equ:prob1a}) and (\ref{equ:prob1b}), we have
\begin{equation}
\label{equ:coeff1a}
c_{\ell} \le e \binom{n-1}{n-1-\ell} \left(\frac{1}{n}\right)^{n- \ell}, \quad \ell=1, \cdots, n/2.
\end{equation}

\underline{For  $\ell=n/2+1, \cdots, n-2$, the coefficient $c_\ell$} is calculated by the shortcut $S_{n-1} \to S_{n/2}$ as follows. Consider 
 ${x}_{n-1} \in S_{n-1}$, then according to  Condition (\ref{equDKCondition1}),
\begin{equation*}  
    c_\ell   \le \frac{p(x_{n-1}, S_{\ell})}{p(x_{n-1}, S_{[0,\ell]})}, \quad \ell=n/2+1, \cdots, n-2.
\end{equation*} Since
$x_{n-1}$ has $n/2-1$ one-valued bits and a state in $S_{\ell}$ has $\ell -n/2$ one-valued bits, the transition from $x_{n-1}$ to $S_{\ell}$ happens only if $n-1-\ell$ zero-valued bits are flipped. Thus,
\begin{equation}
\label{equ:prob2a}
p(x_{n-1}, S_\ell) \le \binom{n/2-1}{n-1-\ell} \left(\frac{1}{n}\right)^{n-1-\ell}.
\end{equation} Since
$x_{n/2}$ has $n/2$ one-valued bits, the transition from $x_{n-1}$ to $S_{n/2}$ happens  if 1 zero-valued bit is flipped and other $n/2-1$ one-valued bits are unchanged. Thus,
\begin{equation}
\label{equ:prob2b} p(x_{n-1}, S_{[0,\ell]}) \ge 
p(x_{n-1}, S_{n/2}) \ge \binom{n/2+1}{1}\frac{1}{n} \left(1-\frac{1}{n}\right)^{n/2-1} \ge \frac{1}{2e}.
\end{equation}

Combining (\ref{equ:prob2a}) and (\ref{equ:prob2b}), we have
\begin{equation}
\label{equ:coeff1b}
c_{\ell} \le 2e \binom{n/2-1}{n-1-\ell} \left(\frac{1}{n}\right)^{n-1-\ell},\quad \ell=n/2+1, \cdots, n-2  .
\end{equation}
{For $\ell=n-1$, we use the trivial estimation $c_{n-1}\le 1$.}

Inserting  (\ref{equ:coeff1a}) and (\ref{equ:coeff1b})  into \eqref{equ:Type-cl-d}, we get a lower bound
\begin{equation*}\small
\begin{split}
    d \le& O(1)+ \sum^{n/2}_{\ell=1}  e \binom{n-1}{n-1-\ell} \left(\frac{1}{n}\right)^{n- \ell} \frac{en}{ \ell} +\sum^{n-2}_{\ell=n/2+1} 2e \binom{n/2-1}{n-1-\ell} \left(\frac{1}{n}\right)^{n-1-\ell}\frac{en}{(\ell-n/2)}\\
    \le& O(1)+ e^2 \sum^{n/2}_{\ell=1}   \frac{(n-1) \cdots (l+1)}{(n-1-\ell)!} \left(\frac{1}{n}\right)^{n-1-\ell} \frac{n}{ \ell} 
    \\&+2e^2\sum^{n-2}_{\ell=n/2+1}   \frac{(n/2-1)\cdots (\ell-n/2+1)}{(n-1-\ell)!} \left(\frac{1}{n}\right)^{n-1-\ell}\frac{n}{(\ell-n/2)}\\
    \le&{O(1)+ e^2 \sum^{n/2}_{\ell=1}   \frac{\ell+1}{(n-1-\ell)! \ell}   +2e^2\sum^{n-2}_{\ell=n/2+1}   \frac{(\ell-n/2+1)}    {(n-1-\ell)!(\ell-n/2)} }
    \\
    \le& {O(1)+ 2e^2 \sum^{n/2}_{\ell=1}   \frac{1}{(n-1-\ell)!}+4e^2 \sum^{n-2}_{\ell=n/2+1}   \frac{1}    {(n-1-\ell)!} =O(1).    } 
\end{split} 
\end{equation*}
The lower bound is $O(1)$, much looser than the actual hitting time $\Theta(n \ln n)$.    

From these two case studies, we see that existing lower bounds are not tight for fitness landscapes with shortcuts. Therefore, it is necessary to improve the lower bound.


\section{Metric bounds from fitness levels}
\label{secDrift}

\subsection{Drift analysis with fitness levels}
{We propose a new method, called drift analysis with fitness levels, which combines fitness levels with drift analysis.}   Its workflow is outlined below. For the sake of illustration, we only present the lower time bound. 
 
First, the search space $S$ is split into multiple fitness levels $(S_0, \cdots, S_K)$ according to the fitness value from high to low, where $S_0=S_{\mathrm{opt}}$.  

Secondly,   states at the same level are   assigned to the same distance from the optimal set, that is, for any $X\in S_0$, $d(X)=0$ and for any $k\ge 1$ and $X \in S_k$, $d(X) =d_k$.  The distance $d_k$  is constructed   using  transition probabilities  between fitness levels.

Next we need to prove that for any $k$ and $X_k \in S_k$, $d_k$ is a lower  bound on the mean hitting time $m(X_k)$. { Since we  consider elitist EAs,} an EA will never move from $X_k \in S_k$ to a fitness level lower than $S_k$. Therefore, the drift  satisfies
 \begin{align}
 \label{equ:ElististDrift}
  \Delta d(X_k) =   
          d_k - \sum^{k}_{\ell=0}  d_{\ell} p(X_k,S_{\ell}) = d_k  p(X_k,S_{[0,k-1]})-  \sum^{k-1}_{\ell=1} d_{\ell}  p(X_,S_{\ell}).
    \end{align}
According to Lemma~\ref{lemmaLowerBound1}, if for any $k\ge 1$ and $X_k \in S_k$, the drift $\Delta d(X_k)\le 1$, then  the hitting time $m(X_k) \ge d(X_k)$.

Finally, the tightest lower bound problem is regarded as a constrained multi-objective optimization problem 
subject to the constraint that $d_k$  is constructed using transition probabilities between fitness levels.

The above drift analysis with fitness levels treats the fitness level method as a special kind of drift analysis. It is completely differs from existing fitness level methods~\citep{wegener2003methods,sudholt2012new,doerr2021lower}.

\subsection{Metric bounds} 
Using transition probabilities between fitness levels,  we construct  $d_k$ recursively  by \eqref{equLowerBound1}.
 Theorem~\ref{theoremLowerBound1} proves that it is a lower bound. It is called \emph{a Type-$r_{k,\ell}$ lower bound}. 
 
\begin{theorem} [Type-$r_{k,\ell}$ lower bound] 
\label{theoremLowerBound1}  Given an elitist EA for maximizing $f(x)$, a fitness level partition $(S_0, \cdots, S_K)$,   probabilities $p(X_k, S_{[0,k-1]})$ and  $r(X_k,S_{\ell})$  (where $1\le \ell<k\le K$), consider the family of  distances  $ (d_1, \cdots, d_k)$ such that for any $X_k \in S_k$, $d(X_k) =d_k$. Then for any $k >0$ and $X_k \in S_k$, the drift
$\Delta d(X_k)\le 1$ if and only if
  \begin{align} 
  \label{equLowerBound1}
     d_k&\le \min_{X_k \in S_k} \left\{\frac{1}{p(X_k, S_{[0,k-1]})} + \sum^{k-1}_{\ell=1} r(X_k,S_{\ell})  d_{\ell}\right\}.
\end{align}
\end{theorem}

According to Lemma~\ref{lemmaLowerBound1}, the drift  $\Delta d(X_k) \le 1$  ensures that $d_k$ is a lower bound. The best lower bound $d^*_k$ is reached when  Inequality (\ref{equLowerBound1}) becomes an equality.

\begin{proof}  First we prove the \underline{sufficient condition.} Suppose that (\ref{equLowerBound1}) is true. 
Since the EA is elitist,  for any $k\ge 1$ and $X_k\in S_k$, by (\ref{equ:ElististDrift}), we have  
 \begin{equation*}
  \Delta d(X_k) =     p(X_k,S_{[0,k-1]})d_k -  \sum^{k-1}_{\ell=1} p(X_k,S_{\ell}) d_{\ell} .
 \end{equation*} We replace  $d_k$ (but not $d_\ell$) with (\ref{equLowerBound1}) and get  
\begin{equation*}  
\begin{split}
     \Delta d(X_k)    
            &\le   p(X_k,S_{[0,k-1]})  \min_{Y_k \in S_k}\left\{\frac{1}{ p(Y_k, S_{[0,k-1]})} + \sum^{k-1}_{\ell=1}r(Y_k,S_{\ell})  d_{\ell} \right\}   - \sum^{k-1}_{\ell=1}  p(X_k,S_{\ell})d_{\ell}\\
            &\le  p(X_k,S_{[0,k-1]})  \left\{\frac{1}{ p(X_k, S_{[0,k-1]})} + \sum^{k-1}_{\ell=1}\frac{p(X_k,S_{\ell})}{p(X_k,S_{[0,k-1]}} d_{\ell} \right\}   - \sum^{k-1}_{\ell=1}  p(X_k,S_{\ell})d_{\ell}\\
            &= 1 +  \sum^{k-1}_{\ell=1} p(X_k,S_{\ell}) d_{\ell}-  \sum^{k-1}_{\ell=1}  p(X_k,S_{\ell})d_{\ell} =1. 
\end{split}
\end{equation*} 
We complete the proof of the sufficient condition.  

Secondly, we prove the \underline{necessary condition.}  Suppose that for any $k\ge 1$ and $X_k\in S_k$, $\Delta d(X_k)\le 1$. 
 Since the EA is elitist,   by (\ref{equ:ElististDrift}),    we have 
$$
  \Delta d(X_k) =   d_k  p(X_k,S_{[0,k-1]})-  \sum^{k-1}_{\ell=1} d_{\ell}  p(X_k,S_{\ell}) \le 1. $$
Then we have 
 \begin{align}  
     d_k\le   \frac{1}{p(X_k, S_{[0,k-1]})} + \sum^{k-1}_{\ell=1} {r(X_k,S_{\ell})}  d_{\ell}. 
\end{align}

{Since the above inequality is true for all $X_k\in S_k$, it is true for the minimum over all $X_k$. We complete the proof of the  {necessary condition.} }
\end{proof}

%

Similarly,   we construct  an upper bound $d_k$ recursively   by \eqref{equUpperBound1}.
 Theorem~\ref{theoremUpperBound1} proves that it is an upper bound.   This  upper bound is named \emph{a Type-$r_{k,\ell}$ upper bound}. 
\begin{theorem} [Type-$r_{k,\ell}$ upper bound] 
\label{theoremUpperBound1}  Given an elitist EA for maximizing $f(x)$, a fitness level partition $(S_0, \cdots, S_K)$,  probabilities $p(X_k, S_{[0,k-1]})$ and  $r(X_k,S_{\ell})$ (where $1\le \ell<k\le K$), consider the family of  distances  $ (d_1, \cdots, d_k)$ such that for any $X_k \in S_k$, $d(X_k) =d_k$. Then for any $k\ge 1$  and $X_k \in S_k$, the drift
$\Delta d(X_k)\ge 1$ if and only if \begin{align} 
\label{equUpperBound1}
    d_k& \ge \max_{X_k \in S_k} \left\{\frac{1}{ p(X_k, S_{[0,k-1]})} + \sum^{k-1}_{\ell=1}  r(X_k,S_{\ell})  d_{\ell}.\right\}
\end{align}  
\end{theorem} 
According to Lemma~\ref{lemmaUpperBound1}, the drift  $\Delta d(X_k) \ge 1$  ensures that $d_k$ is an upper bound.
The best upper bound $d^*_k$ is reached when Inequality (\ref{equUpperBound1}) becomes an equality.  
 %
 
{The main difficulty is how to quickly calculate $d_k$.}
In this paper, we will not  calculate $d_k$ recursively via the metric bounds (\ref{equLowerBound1}) and \eqref{equUpperBound1}. Instead, they are converted to  the linear bounds \eqref{equLinearLower} and \eqref{equLinearUpper} respectively. For example, the upper bound (\ref{equUpperBound1}) is converted to
\begin{align*}
d_k &= \frac{1}{p_{\scriptscriptstyle\min}(X_{k},S_{[0,k-1]})} +   \sum^{k-1}_{\ell=1} r_{\scriptscriptstyle\max}(X_k,S_{\ell})  d_{\ell}, \qquad k=1, \cdots, K. 
\end{align*} Then by induction,  we represent $d_k$ in  a linear form as follows: 
\begin{align*} 
d_k&=\frac{1}{p_{\scriptscriptstyle\min}(X_{k},S_{[0,k-1]})} +\sum^{k-1}_{\ell=1} \frac{c_{k,\ell}}{ p_{\scriptscriptstyle\min}(X_{\ell},S_{[0,\ell-1]})}.
\end{align*}
The problem of calculating  a metric bound becomes the problem of calculating a  linear  bound or  coefficients $c_{k,\ell}$. This will be discussed in detail in the next section.

\subsection{The  tightest  metric bounds}
\label{{secTightest}}
Consider the problem of the tightest lower bound first. Given a family of lower bounds based on a fitness level partition, we want to determine which bound is the tightest. 
A family of lower bounds can be represented by a family of distances in drift analysis. Given a fitness level partition $(S_0, \cdots, S_K)$, consider the family of  distances $(d_0,  \cdots, d_K)$ such that  $d_0=0$ and for any $k\ge 1$, the drift $\Delta d(X_k) \le 1$.    The condition $\Delta d(X_k) \le {1}$ ensures that  $d_k$ is a lower bound on the mean hitting time $m(X_k)$. 

The  tightest lower bound problem is  a constrained multi-objective optimization problem:  
\begin{equation}
    \label{equ:MOP-lower}    
   \max \{d_k;  \Delta d(X_k) \le {1}\}, \quad k=1, \cdots, K,
\end{equation}    
{subject to the constraint that $d_0=0$ and for all $k \ge 1$ and $X_k \in S_k$, $d(X_k) =d_k$.}  According to Theorem~\ref{theoremLowerBound1},  the best lower bound $d^*_k$ by (\ref{equLowerBound1}) is the tightest.   

\begin{theorem}
\label{theoremTightestLowerBound1}  
Given an elitist EA for maximizing $f(x)$, a fitness level partition $(S_0, \cdots, S_K)$,  probabilities $p(X_k, S_{[0,k-1]})$ and  $r(X_k,S_{\ell})$ (where $1\le \ell<k\le K$), consider the family of  distances  $(d_0, d_1, \cdots, d_k)$ such that $d_0=0$ and for all $k \ge 1$ and $X_k \in S_k$, $d(X_k) =d_k$  and the drift $\Delta d (X_k) \le 1$ . 
The  tightest lower bound  within this distance family  is 
 \begin{align} 
 \label{equTightestLower}
     d^*_k=\min_{X_k \in S_k} \left\{\frac{1}{p(X_k, S_{[0,k-1]})} + \sum^{k-1}_{\ell=1}r(X_k,S_{\ell})  d^*_{\ell}\right\}. 
\end{align} 
\end{theorem}

\begin{proof} For $k=1, \cdots, K$, since  $\Delta d(X_k)\le 1$, according to Theorem~\ref{theoremLowerBound1},    we have
 \begin{align*}  
          d_k\le \min_{X_k \in S_k} \left\{\frac{1}{p(X_k, S_{[0,k-1]})} + \sum^{k-1}_{\ell=1}r(X_k,S_{\ell})  d_{\ell}\right\}. 
\end{align*}
The above $d_k$ reaches the maximum when the inequality becomes an equality. 
\end{proof}

Similarly, the tightest upper bound problem   is another constrained multi-objective optimization problem: \begin{equation}
    \label{equ:MOP-upper}    
   \min \{d_k;  \Delta d(X_k) \ge {1}\}, \quad k=1, \cdots, K.
\end{equation}  
{subject to the constrain that $d_0=0$ and for all $k \ge 1$ and $X_k \in S_k$, $d(X_k) =d_k$.} 
According to Theorem~\ref{theoremUpperBound1},  the best upper bound $d^*_k$ by (\ref{equUpperBound1}) is the tightest.   

\begin{theorem}
\label{theoremTightestUpperBound1}  
Given an elitist EA for maximizing $f(x)$, a fitness level partition $(S_0, \cdots, S_K)$,   probabilities $p(X_k, S_{[0,k-1]})$ and  $r(X_k, S_\ell)$ (where $1\le \ell<k\le K$), consider the family of  distances  $ (d_0, \cdots, d_k)$ such that  $d_0=0$ and for $k \ge 1$ and  all $X_k \in S_k$, $d(X_k) =d_k$  and the drift $\Delta d (X_k) \ge 1$.  The tightest upper  bound within the distance family is 
 \begin{align} 
 \label{equTightestUpper}
     d^*_k=\max_{X_k \in S_k} \left\{ \frac{1}{p(X_k, S_{[0,k-1]})} + \sum^{k-1}_{\ell=1} {r(X_k,S_{\ell})}  d^*_{\ell}\right\}.
\end{align} 
\end{theorem}

\section{Linear bounds based on fitness levels}
\label{secLinear}

\subsection{Linear  bounds}  
Based on metric bounds, we constructs linear bounds~\eqref{equLinearLower} and \eqref{equLinearUpper}.   
The theorem below provides coefficients in the lower  bound (\ref{equLinearLower}). 

\begin{theorem} 
\label{theoremLowerBound2}     Given an elitist EA for maximizing $f(x)$,   a fitness level partition $(S_0, \cdots, S_K)$,   probabilities $p_{{\scriptscriptstyle\max}}(X_{\ell},S_{[0,\ell-1]})$ and  $r(X_k, S_\ell)$  (where $1\le \ell<k\le K$), construct $d_k$  as 
\begin{equation} \label{equLinearLowerBound}   
d_k=\frac{1}{p_{{\scriptscriptstyle\max}}(X_{k},S_{[0,k-1]})} +\sum^{k-1}_{\ell=1} \frac{c_{k,\ell}}{ p_{{\scriptscriptstyle\max}}(X_{\ell},S_{[0,\ell-1]})},
\end{equation}  where coefficients $c_{k,\ell} \in [0,1]$ satisfy 
\begin{equation}
\label{equCoefficientLower}
    c_{k,\ell} \le \min_{X_k \in S_k} \left\{ r(X_k, S_\ell) + \sum^{k-1}_{j=\ell+1}  {r}(X_k, S_j) c_{j,\ell} \right\}.
\end{equation}
Then for any $k >0$ and $X_k \in S_k$,  the mean hitting time   
$ m(X_k) \ge  d_k. $
\end{theorem}

\begin{proof}  According to Lemma~\ref{lemmaLowerBound1}, it is sufficient   to prove that for any $k\ge 1$ and $X_k \in S_k$, the drift $\Delta d(X_k) \le 1$.
Since the EA is elitist,   from (\ref{equ:ElististDrift}), we know  
 \begin{equation*}  \label{equDrift2}
 \begin{split}
  \Delta d(X_k) =&     p(X_k,S_{[0,k-1]})d_k -  \sum^{k-1}_{\ell=1} p(X_k,S_{\ell}) d_{\ell} = p (X_k,S_{[0,k-1]})\left(d_k -  \sum^{k-1}_{\ell=1}  r(X_k,S_{\ell}) d_{\ell}\right).
 \end{split}
 \end{equation*} 
 We replace  $d_k$ and $d_\ell$ with (\ref{equLinearLowerBound}) and get  
\begin{equation} \label{equFact1}
    \begin{split}
      \Delta d(X_k) = &    p (X_k,S_{[0,k-1]}) \left[\frac{1}{p_{{\scriptscriptstyle\max}}(X_{k},S_{[0,k-1]})} +\sum^{k-1}_{\ell=1} \frac{c_{k,\ell}}{ p_{{\scriptscriptstyle\max}}(X_{\ell},S_{[0,\ell-1]})}\right.\\
      &\left.-  \sum^{k-1}_{\ell=1} r (X_k,S_{\ell}) \left(\frac{1}{p_{{\scriptscriptstyle\max}}(X_{\ell},S_{[0,\ell-1]})} +\sum^{\ell-1}_{j=1} \frac{c_{\ell,j}}{ p_{{\scriptscriptstyle\max}}(X_{j},S_{[0,j-1]})}\right) \right]. 
    \end{split}  
\end{equation}

In the double summation  $\sum^{k-1}_{\ell=1}\sum^{\ell-1}_{j=1}$, {the  first term for $\ell=1$  is empty because $\ell-1=0<j=1$,}  but it is kept for the sake of notation. We expand this double summation 
 and then merge the same term  $1/p_{{\scriptscriptstyle\max}}(X_{\ell},S_{[0,\ell-1]})$ (where $\ell=1, \cdots, k$) as follows.
\begin{equation}\label{equFact2}
    \begin{split}
 &\sum^{k-1}_{\ell=1}   r(X_k, S_{\ell})  \sum^{\ell-1}_{j=1} \frac{c_{\ell,j}}{ p_{{\scriptscriptstyle\max}}(X_{j},S_{[0,j-1]})} = \sum^{k-1}_{\ell=2}    \sum^{\ell-1}_{j=1} \frac{r(X_k, S_{\ell}) c_{\ell,j}}{ p_{{\scriptscriptstyle\max}}(X_{j},S_{[0,j-1]})}   \\
= &  \frac{r (X_k, S_2)c_{2,1}}{ p_{{\scriptscriptstyle\max}}(X_{1},S_{0})} +  \cdots +\left( \frac{r (X_k, S_{k-1}) c_{k-1,1}}{ p_{{\scriptscriptstyle\max}}(X_{1},S_{0})} + \frac{r (X_k, S_{k-1}) c_{{k-1},{k-2}}}{ p_{{\scriptscriptstyle\max}}(X_{k-2},S_{[0,k-3]})} \right)  \\
=&  \frac{\sum^{k-1}_{j=2} r (X_k, S_j) c_{j,1}}{ p_{{\scriptscriptstyle\max}}(X_{1},S_{0})} +   \cdots +  \frac{\sum^{k-1}_{j=k-1} r (X_k, S_j)c_{j,k-2}}{ p_{{\scriptscriptstyle\max}}(X_{k-2},S_{[0,k-3]})}   \\
= &\sum^{k-2}_{\ell=1}  \frac{\sum^{k-1}_{j=\ell+1}r (X_k, S_j) c_{j,\ell}}{p_{{\scriptscriptstyle\max}}(X_{\ell},S_{[0,\ell-1]})}=\sum^{k-1}_{\ell=1}  \frac{\sum^{k-1}_{j=\ell+1}r (X_k, S_j) c_{j,\ell}}{p_{{\scriptscriptstyle\max}}(X_{\ell},S_{[0,\ell-1]})}.       
    \end{split} 
\end{equation}  
In the double summation $\sum^{k-1}_{\ell=1}\sum^{k-1}_{j=\ell+1}$, the last term for $\ell=k-1$   is empty because $k-1 >\ell+1=k$, but it is added for the sake of notation.  Inserting \eqref{equFact2} into \eqref{equFact1}, we have
\begin{equation}
\label{equ:drift4a}
\begin{split}
      \Delta d(X_k) \le &    p (X_k,S_{[0,k-1]}) \left(\frac{1}{p_{{\scriptscriptstyle\max}}(X_{k},S_{[0,k-1]})} +   \sum^{k-1}_{\ell=1} \frac{c_{k,\ell}}{ p_{{\scriptscriptstyle\max}}(X_{\ell},S_{[0,\ell-1]})}\right.\\
      &\left.-   \sum^{k-1}_{\ell=1}   \frac{r (X_k, S_{\ell})  + \sum^{k-1}_{j=\ell+1} r (X_k, S_j) c_{j,\ell}}{p_{{\scriptscriptstyle\max}}(X_{\ell},S_{[0,\ell-1]})}  \right).  
\end{split}
\end{equation}

According to Condition~(\ref{equCoefficientLower}), 
$c_{k,\ell} \le   r(X_k, S_\ell) + \sum^{k-1}_{j=\ell+1}  {r}(X_k, S_j) c_{j,\ell}$. Inserting it to \eqref{equ:drift4a}, 
we get $\Delta d(X_k) \le 1$ and complete the proof.
\end{proof}

Similarly, 
the theorem below provides coefficients in the upper bound~\eqref{equLinearUpper}.    Its proof is  similar to Theorem~\ref{theoremLowerBound2}.  

\begin{theorem} 
\label{theoremUpperBound2} 
Given an elitist EA for maximizing $f(x)$,     a fitness level partition $(S_0, \cdots, S_K)$, probabilities $p_{{\scriptscriptstyle\min}}(X_{\ell},S_{[0,\ell-1]})$  and  $r(X_k, S_\ell)$   where $1\le \ell<k\le K$, construct $d_k$  as 
\begin{equation}
\label{equLinearUpperBound} 
 d_k =  \frac{1}{p_{{\scriptscriptstyle\min}}(X_{k},S_{[0,k-1]})} +\sum^{k-1}_{\ell=1} \frac{c_{k,\ell}}{ p_{{\scriptscriptstyle\min}}(X_{\ell},S_{[0,\ell-1]})}.
\end{equation}  where coefficients $c_{k,\ell} \in [0,1]$ satisfy 
\begin{equation}
\label{equCoefficientUpper}
    c_{k,\ell} \ge \max_{X_k \in S_k} \left\{ r(X_k, S_\ell) + \sum^{k-1}_{j=\ell+1}  {r}(X_k, S_j) c_{j,\ell} \right\}.
\end{equation}
Then for any $k >0$ and $X_k \in S_k$,  the mean hitting time  
$ m(X_k) \ge  d_k. $
\end{theorem}

The best upper bound $d^*_k$  and its coefficients $c^*_{k,\ell}$ are reached when Inequality (\ref{equCoefficientUpper}) becomes an equality. The best lower bound $d^*_k$  and its coefficients $c^*_{k,\ell}$ are reached when Inequality (\ref{equCoefficientLower}) becomes an equality.   However, it still needs to be rigorously proven under which family of bounds the best  bound is the tightest.

There are three different ways to calculate coefficients $c_{k,\ell}$  by solving Inequality (\ref{equCoefficientLower}) or (\ref{equCoefficientUpper}).
\begin{enumerate}
    \item find an explicit expression of   $c_{k,\ell}$ from  (\ref{equCoefficientLower}) or (\ref{equCoefficientUpper}); 
    \item  recursively calculate  $c_{\ell+1,\ell}$, $c_{\ell+2,\ell}$, $\cdots$, $c_{k,\ell}$ by (\ref{equCoefficientLower}) or (\ref{equCoefficientUpper}); 
    \item combine the above two ways together, that is, for some $c_{k,\ell}$, use recursive calculations; but for other $c_{k,\ell}$, use an explicit expression. 
\end{enumerate}

 
\subsection{Explicit expressions for linear bound coefficients}
An explicit expression of coefficients $c_{k,\ell}$ is more convenient in application.
From \eqref{equCoefficientLower}, by induction, it is straightforward to obtain such an explicit  expression for  lower bound coefficients as follows. 
\begin{equation}
\label{equ:NonrecursiveLow}
\begin{split}
  c_{k,\ell}  \le & 
       r_{{\scriptscriptstyle\min}}(X_k, S_{\ell}) + \sum_{ \ell<j_1<k} r_{{\scriptscriptstyle\min}}(X_k, S_{j_1})\, r_{{\scriptscriptstyle\min}}(X_{j_1}, S_{\ell})   
+\cdots  \\
&+ \sum_{ \ell<j_{k-\ell-1}<\cdots<j_1<k} \, r_{{\scriptscriptstyle\min}}(X_k, S_{j_1}) \, r_{{\scriptscriptstyle\min}}(X_{j_1}, S_{j_2}) \cdots  r_{{\scriptscriptstyle\min}}(X_{j_{k-\ell-1}}, S_{\ell}).  
\end{split}
\end{equation}  

Inequality \eqref{equ:NonrecursiveLow}  provides an intuitive interpretation of the coefficient $c_{k,\ell}$. Each product in \eqref{equ:NonrecursiveLow}   can be interpreted as a conditional  probability of the EA  starting from $X_k$ to visit $S_\ell$.  The coefficient $c_{k,\ell}$ is a lower bound on the sum of these probabilities.

Similarly, from \eqref{equCoefficientUpper}, by induction, it is straightforward to obtain an explicit expression for upper bound coefficients  as follows.  
\begin{equation}
\label{equ:NonrecursiveUp}
\begin{split}
 c_{k,\ell}   \ge &       r_{{\scriptscriptstyle\max}}(X_k, S_{\ell}) + \sum_{ \ell<j_1<k} r_{{\scriptscriptstyle\max}}(X_k, S_{j_1}) \, r_{{\scriptscriptstyle\max}}(X_{j_1}, S_{\ell})    
+\cdots  \\
     &+ \sum_{ \ell<j_{k-\ell-1}<\cdots<j_1<k} r_{{\scriptscriptstyle\max}}(X_k, S_{j_1}) \, r_{{\scriptscriptstyle\max}}(X_{j_1}, S_{j_2}) \cdots  r_{{\scriptscriptstyle\max}}(X_{j_{k-\ell-1}}, S_{\ell}).  
\end{split}
\end{equation}

The number of summation terms in \eqref{equ:NonrecursiveLow} and \eqref{equ:NonrecursiveUp} is up to $(k-\ell-1)!$. Therefore, it is intractable to calculate coefficients by \eqref{equ:NonrecursiveLow} and \eqref{equ:NonrecursiveUp}. But there are many ways to construct explicit expressions that can be calculated in polynomial time. For example,  for the lower bound,  a simple  expression from \eqref{equ:NonrecursiveLow} is  $c_{k,\ell}  \le   r_{{\scriptscriptstyle\min}}(X_k, S_{\ell})$.  Recently \cite{he2023fast} propose a simplified version of \eqref{equ:NonrecursiveLow} as shown below, which can be used to obtain tight lower bounds on  fitness landscapes with  shortcuts. 
 \begin{equation} 
 \label{equ:Simple-lower}
c_{k,l}\le \prod_{i \in [\ell+1,k]}   r_{{\scriptscriptstyle\min}}(X_i,S_{[\ell,i-1]}).  
\end{equation}  
For the upper bound,  an explicit  expression from \eqref{equ:NonrecursiveUp} is  $c_{k,\ell}  \ge   r_{{\scriptscriptstyle\max}}(X_k, S_{[\ell,k-1]})$.
 
Another simple way is to assign $c_{k, \ell} =0,1,$ $ c,$ or $c_\ell$.   Although Type-$0,1$, $c$ and $c_{\ell}$ bounds have been studied in~\citep{wegener2003methods,sudholt2012new,doerr2021lower},  our proof is  completely different. Furthermore,  our Type-$c$ upper and Type-$c_{\ell}$ lower bounds require  weaker conditions. Hence, they are not exactly the same as  those in \citep{sudholt2012new,doerr2021lower}.  

Let $c_{k,\ell}=0$, then  the linear lower bound (\ref{equLinearLowerBound}) becomes the same Type-$0$ lower bound  as Proposition~\ref{proposition1}.   

\begin{corollary}[Type-$0$ lower bound]  
\label{corollaryLowerBound-0} 
For $1\le \ell <k\le K$,
choose $c_{k,\ell} = 0$, then   
$ m(X_k) \ge \frac{1}{p_{{\scriptscriptstyle\max}}(X_{k},S_{[0,k-1]})}. 
$  
\end{corollary}

Let $c_{k,\ell}=c$, then the linear lower bound (\ref{equLinearLowerBound}) becomes a Type-$c$ lower bound.

\begin{corollary}[Type-$c$ lower bound]  
\label{corollaryLowerBound-c} 
For $1\le \ell <k\le K$,
choose  $c_{k,\ell}=c$ to satisfy the inequality
\begin{equation} 
\label{equLowerCoeff-c} 
c \le \min_{1 < k\le K} \min_{1\le \ell<k}  \min_{X_k: p(X_k, S_{[0,\ell]})>0} \frac{p(X_k, S_{\ell})}{p(X_k, S_{[0,\ell]})}.
\end{equation}  
 Then   
 $ 
m(X_k) \ge \frac{1}{p_{{\scriptscriptstyle\max}}(X_{k},S_{[0,k-1]})} +\sum^{k-1}_{\ell=1} \frac{c }{ p_{{\scriptscriptstyle\max}}(X_{\ell},S_{[0,\ell-1]})}  . 
$   
\end{corollary}     

\begin{proof} Condition (\ref{equLowerCoeff-c}) is equivalent to that for any $1\le \ell <k\le K$ and $X_k\in S_k$ such that $p(X_k, S_{[0,\ell-1]})>0$,  it holds
\begin{align} 
&c \le \frac{p(X_k, S_{\ell})}{p(X_k, S_{[0,\ell]})}=  \frac{r(X_k, S_{\ell})}{r(X_k, S_{[0,\ell]})}.\nonumber\\
&c\, r(X_k, S_{[0,\ell]}) =c\, (1-r(X_k, S_{[\ell+1,k-1]})) \le   r(X_k, S_{\ell}).  \label{equLowerCoeff-c-1c}\\
&c \le  r(X_k, S_{\ell}) + c  \,  r(X_k, S_{[\ell+1,k-1]}). 
\label{equLowerCoeff-c-1a} 
\end{align} 

For any $1\le \ell <k\le K$ and $X_k\in S_k$ such that $p(X_k, S_{[0,\ell-1]})=0$, we have $r(X_k, S_{[0,\ell]})=0$  and $r(X_k, S_{[\ell+1,k-1]})=1$. Thus we get an identical equation:  
\begin{align} 
c =  r(X_k, S_{\ell}) + c   \, r(X_k, S_{[\ell+1,k-1]})=0+c. 
\label{equLowerCoeff-c-1b} 
\end{align} 

Combining \eqref{equLowerCoeff-c-1a} and \eqref{equLowerCoeff-c-1b}, we get Condition \eqref{equCoefficientLower}, that is  for all $X_k \in S_k$,
\begin{equation*} 
    c \le   r (X_k, S_\ell) + c \sum^{k-1}_{j=\ell+1}  r(X_k, S_{\ell}) .
\end{equation*} The above inequality is true for the minimum over all $X_k$. The corollary is derived from Theorem~\ref{theoremLowerBound2}.
\end{proof}

Corollary~\ref{corollaryLowerBound-c}  provides  an interpretation of the coefficient $c$ which is a lower bound on the conditional probability (\ref{equLowerCoeff-c}).   Corollary~\ref{corollaryLowerBound-c} is more convenient than Proposition~\ref{proposition3} because   the coefficient $c$ is calculated directly from probabilities ${p(X_k, S_{\ell})}$ and $ {p(X_k, S_{[0,\ell]})}$.  Inequality (\ref{equLowerCoeff-c-1c}) is  equivalent to  the inequality  ${r_{k,\ell}} \ge c {\sum^{\ell}_{j=0} r_{k,j}}$ in Proposition~\ref{proposition3} under different representations, so  Corollary~\ref{corollaryLowerBound-c} is equivalent to  Proposition~\ref{proposition3}.

Let   $c_{k,\ell}=c_{\ell}$, then   the linear lower bound (\ref{equLinearLowerBound}) becomes a Type-$c_\ell$ lower bound. The proof of Corollary~\ref{corollaryLowerBound-cl} is similar to Corollary~\ref{corollaryLowerBound-c} so we omit its proof.

\begin{corollary}[Type-$c_\ell$ lower bound]  
\label{corollaryLowerBound-cl} 
For $1\le \ell <k\le K$,
choose  $c_{k,\ell}=c_\ell$ to satisfy the inequality
\begin{equation} 
\label{equLowerCoeff-cl} 
    c_{\ell} \le \min_{\ell < k\le K}     \min_{X_k: p(X_k, S_{[0,\ell]})>0} \frac{p(X_k, S_{\ell})}{p(X_k, S_{[0,\ell]})}.
\end{equation}  
 Then   
 $ 
m(X_k) \ge \frac{1}{p_{{\scriptscriptstyle\max}}(X_{k},S_{[0,k-1]})} +\sum^{k-1}_{\ell=1} \frac{c_\ell }{ p_{{\scriptscriptstyle\max}}(X_{\ell},S_{[0,\ell-1]})}. 
$  
\end{corollary}    

{
The above corollary is similar to  Lemma~\ref{lemmaVisitProbability}, but  it  does not require the initialization condition \eqref{equDKCondition2} in Lemma~\ref{lemmaVisitProbability}. Corollary~\ref{corollaryLowerBound-cl} can be used to handle any random initialization by replacing $m(X_k)$ with $m(X^{[0]})$ such that $$m(X^{[0]}) \ge \sum^K_{k=1}  \Pr(X^{[0]} \in S_k) \left(\frac{1}{p_{{\scriptscriptstyle\max}}(X_{k},S_{[0,k-1]})} +\sum^{k-1}_{\ell=1} \frac{c_\ell}{ p_{{\scriptscriptstyle\max}}(X_{\ell},S_{[0,\ell-1]})}\right).$$}

Similarly, let  $c_{k,\ell}=1$, then  linear upper bound (\ref{equLinearUpperBound}) becomes the same Type-$1$ bound as in  Proposition~\ref{proposition2}.
 
\begin{corollary}[Type-$1$ upper bound]  
\label{corollaryUpperBound-1} For $1\le \ell <k\le K$,
choose $c_{k,\ell} = 1$, then  
$ 
m(X_k) \le \sum^{k}_{\ell=1} \frac{1}{p_{{\scriptscriptstyle\min}}(X_{\ell},S_{[0,\ell-1]})}. 
$  
\end{corollary}  

Let  $c_{k,\ell}=c$, then  the linear lower bound (\ref{equLinearUpperBound}) becomes a Type-$c$ upper bound.  The proof of Corollary~\ref{corollaryUpperBound-c} is similar to Corollary~\ref{corollaryLowerBound-c}.  We omit its proof since we only need to replace the minimum with the maximum.
 \begin{corollary}[Type-$c$ upper bound]  
\label{corollaryUpperBound-c} 
For $1\le \ell <k\le K$,
choose $c_{k,\ell}=c$ to satisfy  
\begin{equation}  
\label{equUpperCoeff-c} 
 c \ge  \max_{1<k\le K} \max_{1 \le \ell<k}   \max_{X_k: p(X_k, S_{[0,\ell]})>0} \frac{p(X_k, S_{\ell})}{p(X_k, S_{[0,\ell]})}.
\end{equation}  
 Then 
 $ 
m(X_k) \le \frac{1}{p_{{\scriptscriptstyle\max}}(X_{k},S_{[0,k-1]})} +\sum^{k-1}_{\ell=1} \frac{c }{ p_{{\scriptscriptstyle\max}}(X_{\ell},S_{[0,\ell-1]})}  . 
$   
\end{corollary}

Corollary~\ref{corollaryUpperBound-c} is  more convenient  than Proposition~\ref{proposition4} because  the coefficient $c$ is calculated directly from transition probabilities ${p(X_k, S_{\ell})}$ and $ {p(X_k, S_{[0,\ell]})}$. But unlike Proposition~\ref{proposition4}, Corollary~\ref{corollaryUpperBound-c} does not require the condition that  for all $1 \le l \le K - 2$,  $(1 - c)p_{{\scriptscriptstyle\min}}(X_{\ell+1},S_{[0,\ell]}) \le  p_{{\scriptscriptstyle\min}}(X_{\ell},S_{[0,\ell-1]})$. Therefore,   Proposition~\ref{proposition4} is a special case of Corollary~\ref{corollaryUpperBound-c}.

Let  $c_{k,\ell}=c_{\ell}$, then the linear upper bound (\ref{equLinearUpperBound}) becomes a Type-$c_\ell$ upper bound.  The proof of Corollary~\ref{corollaryUpperBound-cl} is similar to Corollary~\ref{corollaryLowerBound-c}.

\begin{corollary} [Type-$c_\ell$ upper bound] 
\label{corollaryUpperBound-cl} For $1\le \ell <k\le K$,
choose  $c_{k,\ell} = c_{\ell}$ to satisfy  
\begin{equation} 
\label{equUpperCoeff-cl}
c_\ell \ge \max_{\ell<k<K}   \max_{X_k: p(X_k, S_{[0,\ell]})>0} \frac{p(X_k, S_{\ell})}{p(X_k, S_{[0,\ell]})}.
\end{equation}  
then    $
m(X_k) \le \frac{1}{p_{{\scriptscriptstyle\min}}(X_{k},S_{[0,k-1]})} +\sum^{k-1}_{\ell=1} \frac{c_{\ell} }{ p_{{\scriptscriptstyle\min}}(X_{\ell},S_{[0,\ell-1]})}. 
$ 
\end{corollary} 
Corollary~\ref{corollaryUpperBound-cl} is new and completely different from   Proposition~\ref{proposition6}. The coefficient $c_\ell$ is directly calculated  from transition probabilities ${p(X_k, S_{\ell})}$ and $ {p(X_k, S_{[0,\ell]})}$.  

\subsection{Discussion of different linear bounds}
As shown in Corollaries 1 to 6, Types-$0,1,c,c_\ell$ bounds are   special cases of  Type-$c_{k,\ell}$ bound.  Therefore, Type-$c_{k, \ell}$ bound is the best  among them.
\begin{corollary}
\label{corollaryComparison} Given an elitist EA and a fitness level partition $(S_0, \cdots, S_K)$, the best Type-$c_{k,\ell}$ bound is no worse than the best Type-$c$ and Type-$c_\ell$ bounds for lower and upper bounds.  
\end{corollary}


The best Type-$c_{k,\ell}$ bound is the exact hitting time of  EAs  on  level-based fitness landscapes, but the best Type-$c$ and $c_\ell$ bounds usually are not.

\begin{definition}  Given an elitist EA for maximizing a function $f(x)$ and a fitness level partition $(S_0, \cdots, S_K)$,   the  fitness landscape  is called  \emph{level-based landscape}   if for all $1\le \ell <k\le K$ and $X_k \in S_k$,
$p_{\scriptscriptstyle\min}(X_k, S_{\ell}) = p_{\scriptscriptstyle\max}(X_k, S_{\ell})$.  This function is called \emph{level-based function}.
\end{definition}

Both OneMax and TwoMax1 are level-based fitness landscapes for the (1+1) EA. Corollary~\ref{corollaryExactTime} follows directly from Theorem 5 and Theorem 6.

\begin{corollary} 
\label{corollaryExactTime}
Given an elitist EA for maximizing a function $f(x)$ and a fitness level partition $(S_0, \cdots, S_K)$,  if the fitness landscape  is level-based, then  
the best Type-$c_{k,\ell}$ lower bound  is equal to the best Type-$c_{k,\ell}$  upper bound.   
\end{corollary}


In addition, Type-$c$ and Type-$c_\ell$ lower bounds  are loose on fitness landscapes with shortcuts because shortcuts results in   coefficients  $c$ and $c_\ell$ as small as $o(1)$. This claim has been verified in Section~\ref{secTwoMax1-c} and Section~\ref{secTwoMax1-cl}. We obtain it more generally as follows.

\begin{theorem}
\label{theorem:shortcuts}
If a shortcut exists, that is, for some $1\le \ell <k \le K$ and $X_k \in S_k$, it holds 
\begin{equation}
\label{equ:Shortcuts2}
\frac{p(X_k, S_{\ell})}{p(X_k, S_{[0,\ell]})}=o(1),
\end{equation} then coefficients $c = o(1) $ in (\ref{equLowerCoeff-c}) and $c_{\ell} =o(1)$ in (\ref{equLowerCoeff-cl}).      
\end{theorem} 

\begin{proof}
    According to Corollary~\ref{corollaryLowerBound-c}, the lower bound coefficient 
    \begin{equation*}  
c \le \min_{1 < k\le K} \min_{1\le \ell<k}  \min_{X_k: p(X_k, S_{[0,\ell]})>0} \frac{p(X_k, S_{\ell})}{p(X_k, S_{[0,\ell]})}.
\end{equation*}  
By Condition~\eqref{equ:Shortcuts2}, we get  $c = o(1) $. 

According to Corollary~\ref{corollaryLowerBound-cl}, the lower bound coefficient 
    \begin{equation*}  
    c_{\ell} \le \min_{\ell < k\le K}     \min_{X_k: p(X_k, S_{[0,\ell]})>0} \frac{p(X_k, S_{\ell})}{p(X_k, S_{[0,\ell]})}.
\end{equation*}  
By Condition~\eqref{equ:Shortcuts2}, we get  $c_\ell = o(1) $. 
\end{proof}

\section{Applications of the linear lower bound}
\label{secCase} 
\subsection{Case Study 3: Calculating lower bound coefficients for the (1+1) EA on OneMax}
\label{secCoefficients}
In this case study, we demonstrate  different ways to calculate  coefficients in the linear lower bound. Consider the (1+1) EA that maximizes OneMax.  
Assume that the EA starts from $S_n$. According to Theorem~\ref{theoremLowerBound2} and Corollary~\ref{corollaryExactTime}, we consider a Type-$c_{k,\ell}$ lower bound  \[d_n = \frac{1}{p_{ }({x}_{n},S_{[0,n-1]})} +\sum^{n-1}_{\ell=1} \frac{c_{n,\ell}}{p_{ }({x}_{\ell},S_{[0,\ell-1]})}. \] where   
\begin{equation}
\label{equOneMaxCoef}
    c_{n, \ell} \le r(x_n, S_\ell) +\sum^{n-1}_{k=\ell+1} r(x_n, S_k) c_{k,\ell}, \quad \ell=1, \cdots, n-1.
\end{equation}
Since OneMax is a level-based fitness landscape to  the (1+1) EA, the best Type-$c_{k,\ell}$  bound is the exact hitting time. By (\ref{equ:NonrecursiveLow}), we obtain \underline{the best coefficient}  
\begin{equation}
\label{equ:NonrecursiveLow2}
\begin{split}
  c^*_{n,\ell}   = & 
       r (X_n, S_{\ell}) + \sum_{\ell<j_1<n} r (X_n, S_{j_1}) \, r (X_{j_1}, S_{\ell})  
+\cdots  \\
     &+ \sum_{ \ell<j_{n-\ell-1}<\cdots<j_1<n} r(X_n, S_{j_1}) \, r (X_{j_1}, S_{j_2}) \cdots  r (X_{j_{n-\ell-1}}, S_{\ell}).    
\end{split}
\end{equation}  
 Unfortunately the calculation of \eqref{equ:NonrecursiveLow2}  is intractable. Instead,
{ we aim to obtain large coefficients $c_{k,\ell}=\Omega(1)$, but not exact coefficients or exact constants in $\Omega(1)$. It is sufficient to obtain a tight bound using $c_{k,\ell}=\Omega(1)$. 
} 

There are two approaches to calculate $c_{k,\ell}$ via Inequality (\ref{equOneMaxCoef}). One is to look for explicit  solutions to Inequality (\ref{equOneMaxCoef}). The other is recursive calculations level by level.
There are different explicit  solutions to Inequality (\ref{equOneMaxCoef}).
It is trivial to get \underline{the  trivial solution}  $c_{n, \ell}=0$  (where $1\le \ell \le n-1$). From (\ref{equ:NonrecursiveLow2}), it is straightforward to obtain an  \underline{explicit solution} 
\[c_{n, \ell}= r (x_n, S_\ell) = \frac{p(x_n, S_\ell)}{ p(x_n, S_{[0,n-1]})}.\] 
Since $p(x_n, S_\ell)\ge \binom{n}{n-\ell}\left(\frac{1}{n}\right)^{n-\ell}\left(1-\frac{1}{n}\right)^{\ell} $ and $ p(x_n, S_{[0,n-1]}) \le 1,
$ we have 
\[c_{n, \ell}=\frac{p(x_n, S_\ell)}{ p(x_n, S_{[0,n-1]})}\ge  \binom{n}{n-\ell}\left(\frac{1}{n}\right)^{n-\ell}\left(1-\frac{1}{n}\right)^{\ell}.\] 
{ These coefficients are not good enough because it does not prove that $c_{n,1}= \Omega(1)$.}

\underline{A non-trivial explicit  solution}   is to let $c_{k, \ell}=c$.
According to Corollary~\ref{corollaryLowerBound-c},  we choose
\[
  c  =  \min_{1<k \le n} \min_{ 1\le \ell <k}  \frac{p(x_k, S_\ell)}{ p(x_k, S_{[0, \ell]})}.
\]

The transition probability $p(x_k, S_\ell)$   where $\ell=1, \cdots,k-1$ is  calculated as follows. 
Since a state in $S_k$ has $k$ zero-valued bits, the transition from $x_k$ to $S_\ell$ happens if $k-\ell$ zero-valued bits are flipped and other bits unchanged. Thus
\begin{equation}
    \label{equOneMax1a}
    p(x_k, S_\ell) \ge \binom{k}{k-\ell}\left(\frac{1}{n}\right)^{k-\ell} \left(1-\frac{1}{n}\right)^{n-k+\ell}
.\end{equation}
The transition from $x_k$ to $S_{[0,\ell]}$ happens only if $k-\ell$ zero-valued bits are flipped.  Thus
\begin{equation}
    \label{equOneMax1b}
   p(x_k, S_{[0,\ell]}) \le \binom{k}{k-\ell}\left(\frac{1}{n}\right)^{k-\ell} .
 \end{equation} 
{ The upper bound of $p(x_k, S_{[0,\ell]})$ above is not tight, but it is enough to obtain large coefficients $c_{n,\ell}$.}
By (\ref{equOneMax1a}) and (\ref{equOneMax1b}), we have
\[ c= \min_{1<k \le n} \min_{ 1\le \ell <k}  \frac{p(x_k, S_\ell)}{ p(x_k, S_{[0, \ell]})} \ge   \min_{1<k \le n} \min_{ 1\le \ell <k}  \left(1-\frac{1}{n}\right)^{n-k+\ell}\ge  e^{-1}.\] 
Therefore the coefficient $c_{n,\ell}=c$ is as large as $ \Omega(1)$.

\underline{Another non-trivial explicit  solution}     is to let $c_{k, \ell}=c_\ell$.
According to Corollary~\ref{corollaryLowerBound-cl}, we choose the coefficient
\[
  c_\ell  =  \min_{ \ell <k\le n} \frac{p(x_k, S_\ell)}{ p(x_k, S_{[0, \ell]})} .
\]
By (\ref{equOneMax1a}) and (\ref{equOneMax1b}), we have
\[ c_\ell  = \min_{ \ell <k\le n}\frac{p(x_k, S_\ell)}{ p(x_k, S_{[0, \ell]})} \ge \min_{ \ell <k\le n} \left(1-\frac{1}{n}\right)^{n-k+\ell}  \ge e^{-1}.\]
Therefore the coefficient $c_{n,\ell}=c_\ell$ is as large as $ \Omega(1)$.



Inequality (\ref{equOneMaxCoef}) can be solved recursively from $k =\ell+1$ to $n$. \underline{A  recursive solution} to (\ref{equOneMaxCoef}) is calculated as follows. 
According to (\ref{equCoefficientLower}), we choose  the coefficient
\begin{equation*}
    c_{\ell+1, \ell} = r(x_{\ell+1}, S_\ell) =\frac{p(x_{\ell+1}, S_\ell)}{p(x_{\ell+1}, S_{[0,\ell]}) } \ge \left(1-\frac{1}{n}\right)^{n-1}   \ge e^{-1} \quad (\mbox{by (\ref{equOneMax1a}) and (\ref{equOneMax1b})}).  
\end{equation*}
Assume that $c_{\ell+1,\ell}, \cdots, c_{k-1,\ell} \ge e^{-1}$. 
According to (\ref{equCoefficientLower}), we choose the coefficient
\begin{align*}
c_{k, \ell} &
= r(x_k, S_\ell) +\sum^{k-1}_{j=\ell+1} r(x_k, S_j) c_{j,\ell} \\
&\ge r(x_k, S_\ell) +  r(x_k, S_{[\ell+1,k-1]}) e^{-1}=r(x_k, S_\ell)  - r(x_k, S_{[0,\ell]}) e^{-1}+e^{-1}.
\end{align*}
We prove the coefficient $r(x_k, S_\ell)  - r(x_k, S_{[0,\ell]}) e^{-1} \ge 0$  or equivalently $p(x_k, S_\ell)  - rp(x_k, S_{[0,\ell]}) e^{-1} \ge 0$ as follows. 
\begin{align*} 
    & p(x_k, S_\ell)  - p(x_k, S_{[0,\ell]}) e^{-1}
    \\
    \ge& \binom{k}{k-\ell}\left(\frac{1}{n}\right)^{k-\ell} \left(1-\frac{1}{n}\right)^{n-k+\ell} -  \binom{k}{k-\ell}\left(\frac{1}{n}\right)^{k-\ell} e^{-1}     \ge   0.  \quad (\mbox{by (\ref{equOneMax1a}) and (\ref{equOneMax1b})})  
\end{align*}
By induction,     
$    c_{k, \ell} \ge e^{-1}$ for $k=\ell+1, \cdots, n$. Thus the coefficient $c_{n,\ell}$ is as large as $ \Omega(1)$.


Finally, it is allowed to use a mixture of recursive and explicit  solutions, that is, some coefficients are calculated recursively and some coefficients come from an explicit   solution. For example, \underline{a mix of explicit  and recursive solutions} is 
\begin{align}   
    &c_{k,\ell}=c \le r(x_k, S_\ell) +\sum^{k-1}_{j=\ell+1} r(x_k, S_j) c, \qquad 1\le \ell<k\le n-1.\label{equOneMax1c}\\
 &c_{n, \ell}  
= r(x_n, S_\ell) +\sum^{n-1}_{j=\ell+1} r(x_n, S_j) c, \qquad 1\le \ell \le n-1.  
\label{equOneMax1d}
\end{align}   
Similar to  the analysis of the explicit  solution $c$, for (\ref{equOneMax1c}), we get $c =\Omega(1)$. 
In (\ref{equOneMax1d}), 
\[
r(x_n, S_j)  = \frac{p(x_n,S_j)}{p(x_n, S_{[0,n-1]})} \ge \frac{p(x_n,S_j)}{1} \ge \binom{n}{n-j}\left(\frac{1}{n}\right)^{n-j}\left(1-\frac{1}{n}\right)^{j} \quad (\mbox{by } (\ref{equOneMax1a})),
\] then  we get for $\ell=1, \cdots, n-1$
\begin{equation*}  \small   
 c_{n, \ell}  
\ge \binom{n}{n-\ell}\left(\frac{1}{n}\right)^{n-\ell}\left(1-\frac{1}{n}\right)^{\ell} + e^{-1}\sum^{n-1}_{j=\ell+1} \binom{n}{n-j}\left(\frac{1}{n}\right)^{n-j}\left(1-\frac{1}{n}\right)^{j}.  
\end{equation*}   
Thus we get coefficients $c_{n,\ell}= \Omega(1)$. 

In summary, there exist different ways to calculate coefficients $c_{k,\ell}$   from a trivial coefficient $0$ to the exact coefficients $c^*_{k,\ell}$.   Drift analysis with fitness levels is a  framework that can be used to develop different fitness level methods. For example, existing  fitness level methods~\citep{wegener2003methods,sudholt2012new,doerr2021lower,he2023fast} are special cases within the framework.

\subsection{Case Study 4: A tight Type-$c_{k,\ell}$ lower bound of the (1+1) on TwoMax1}
\label{secTwoMax1-ckl} 
In this case study, we show  that the Type-$c_{k,\ell}$ lower bound  by Theorem~\ref{theoremLowerBound2} is tight on fitness landscapes with shortcuts. Consider the (1+1) EA maximizing TwoMax1. Assume that the EA starts from $S_{n-1}$. We prove that the Type-$c_{k,\ell}$ lower bound  by Theorem~\ref{theoremLowerBound2} is $\Omega(n \ln n)$, that is,
\begin{align}\label{equ:TwoMax1-Type-ckl}
   d_{n-1}  \ge
   \sum^{n/2-1}_{\ell=1}\frac{c_{n-1, \ell}}{p({x}_{\ell},S_{[0,\ell-1]})}=\Omega(n \ln n).    
\end{align} 

\underline{The transition probability  ${p({x}_{\ell},S_{[0,\ell-1]})}$ (where  $1\le \ell\le n/2$)}  is calculated as follows. Since a state in $S_\ell$  has $\ell$ zero-valued bits.  
The transition  from $x_\ell$ to $S_{[0,\ell-1]}$ happens only if either (i) 1 zero-valued bit is flipped, or (ii) $x_\ell$ is mutated to $(0,\cdots, 0)$.  The probability of the first event is $ \binom{\ell}{1}\frac{1}{n}$.  The probability of  the second event happening is  $\left(\frac{1}{n}\right)^{n-\ell} \left(1-\frac{1}{n}\right)^\ell$.  Thus  
\begin{equation} 
 p({x}_{\ell},S_{[0,\ell-1]})     \le  \frac{\ell}{n} + \left(\frac{1}{n}\right)^{n-\ell} \le  \frac{\ell+1}{n}. 
\label{equ:uk}  
\end{equation}  
Then  we get a lower bound
\begin{equation} 
\label{equ:TwopathLower1} 
d_{n-1} \ge  \sum^{n/2-1}_{\ell=1} c_{n-1, \ell} \frac{ n}{\ell+1} .
\end{equation}

According to Theorem~\ref{theoremLowerBound2},  we choose  coefficients  
\begin{equation}
\label{equ:lower-coeff}
    c_{n-1,\ell}   =  
    \sum^{n/2}_{k=\ell+1}  r(x_{n-1}, S_k) c_{k,\ell},
    \quad \ell =1, \cdots, \frac{n}{2}-1.     
\end{equation}
 
\underline{The coefficient $c_{k,\ell}$ (where  $1\le \ell<k\le n/2$)} is calculated using a constant $c$.    According to Theorem~\ref{theoremLowerBound2}, for $1\le \ell<k\le n/2$, we choose coefficients $c_{k,\ell} =c$ such that
\begin{align*} 
   & c \le r(x_k, S_{\ell}) + \sum^{k-1}_{i=\ell+1}  r(x_k, S_i) c, \\
   &  c\le   \frac{r(x_k, S_{\ell})}{1-r(x_k, S_{[ \ell+1,k-1]})}= \frac{r(x_k, S_{\ell})}{r(x_k, S_{[0,\ell]})}=\frac{p(x_k, S_{\ell})}{p(x_k, S_{[0,\ell]})}.
\end{align*} 
The above inequality is true for all for $1\le \ell<k\le n/2$, thus we choose   $c $ as
\begin{align}  
     & c =\min_{1<k\le n/2} \min_{1\le \ell<k}  \frac{p(x_k, S_{\ell})}{p(x_k, S_{[0,\ell]})}.
     \label{equ:case-study-4-c} 
\end{align} 

\underline{The above $c$}  is calculated as follows. Since a state in $S_k$ (where $1\le \ell<k\le n/2$) has $k$ zero-valued bits, the transition   from $x_k$ to $S_\ell$  happens if $k-\ell$ zero-valued bits are flipped and other bits are not flipped.   
Thus  
\begin{align}
\label{equTwoMax11a}
  p( x_k,S_\ell) \ge \binom{k}{k-\ell} \left(\frac{1}{n}\right)^{k-\ell}   \left(1-\frac{1}{n}\right)^{n-k+\ell} \ge \binom{k}{k-l} \left(\frac{1}{n}\right)^{k-\ell} e^{-1}. 
\end{align}

The transition  from $x_k$ to $S_{[0, \ell]}$ happens only if either (i)  $k-\ell$ zero-valued bits are flipped, or (ii) $x_k$ is mutated to $(0, \cdots, 0)$.  The probability of the first event is  $\binom{k}{k-\ell} (\frac{1}{n})^{k-\ell}$.  The probability of  the second event  is  $\left(\frac{1}{n}\right)^{n-k} \left(1-\frac{1}{n}\right)^k$. Because $1\le \ell < k \le n/2$, we have $n-k\ge n/2 \ge k-\ell+1$. Thus 
\begin{align}
 p(x_k,S_{[0,\ell]}) &\le   \binom{k}{k-\ell} \left(\frac{1}{n}\right)^{k-\ell}+\left(\frac{1}{n}\right)^{n-k}  \le \binom{k}{k-\ell} \left(\frac{1}{n}\right)^{k-\ell}+ \left(\frac{1}{n}\right)^{k-\ell+1}. 
\label{equTwoMax11b} 
\end{align} 

By (\ref{equTwoMax11a}) and (\ref{equTwoMax11b}) , we get  
\begin{align*}
{ c =\min_{1<k\le n/2} \min_{1\le \ell<k} \frac{p( x_k,S_\ell)}{p(x_k,S_{[0, \ell]})} \ge \frac{\binom{k}{k-l} \left(\frac{1}{n}\right)^{k-\ell} e^{-1}}{\binom{k}{k-\ell} \left(\frac{1}{n}\right)^{k-\ell}+ \left(\frac{1}{n}\right)^{k-\ell+1}} = \frac{\binom{k}{k-\ell}e^{-1}}{\binom{k}{k-\ell}+ \frac{1}{n}} = \Omega(1). }
\end{align*}

Next \underline{the coefficient $c_{n-1,\ell}$  (where $1\le \ell<n/2$)} is calculated by (\ref{equ:lower-coeff}). By (\ref{equ:lower-coeff}), we get for $\ell=1, \cdots, n/2$,
\begin{equation}  
\label{equ:coeff4}
    c_{n-1,\ell} \ge\Omega(1) r(x_{n-1}, S_{[\ell+1,n/2]}) .
\end{equation} 
The conditional probability $r(x_{n-1},S_{[\ell+1,n/2]})$  is calculated as follows. Since a state in $S_{n-1}$ has $n/2+1$ zero-valued bits, the transition from $x_{n-1}$ to $S_{n/2}$ happens if 1 zero-valued bit is flipped and other bits are unchanged. Thus 
\begin{align*}
&p(x_{n-1}, S_{[\ell+1,n/2]}) \ge \binom{n/2+1}{1} \frac{1}{n} \left(1-\frac{1}{n}\right)^{n-1} \ge\frac{1}{2e}. 
\end{align*}
Then we get
\begin{align*} 
r(x_{n-1},S_{[\ell,n/2]}) =\frac{p(x_{n-1},S_{[\ell,n/2]})}{p(x_{n-1},S_{[0,n-2]})}\ge p(x_{n-1},S_{[\ell,n/2]})\ge\frac{1}{2e}.
\end{align*}
By (\ref{equ:coeff4}),   we get $c_{n-1,\ell} \ge \Omega(1)$.  Then by \eqref{equ:TwopathLower1}, we get a lower bound
\begin{equation} 
\label{equ:TwopathLower} 
d_{n-1} \ge  \Omega(1)\sum^{n/2-1}_{\ell=1} \frac{ n}{\ell+1} = \Omega (n \ln n).
\end{equation}   

Table~\ref{tab:comparision} compares Type-$c$, $c_\ell$ and  $c_{k,\ell`}$ lower bounds.  Type~$c_{k,\ell}$ lower bound  is  tight  because the actual hitting time is $\Theta(n\ln n)$. 
But Type-$c$ and $c_\ell$ lower bounds are trivial.

\begin{table}[ht]
    \centering
    \caption{Comparison of different types of lower  bounds of the (1+1) EA on TwoMax1}
    \label{tab:comparision}
    \begin{tabular}{|c|c|c|c|}
    \hline
          Type-$c_{k, \ell}$    & Type-$c$   & Type-$c_l$  \\\hline 
                     $\Omega(n \ln n)$ & $O(1)$ & $O(1)$ \\
        \hline
                     by Theorem~\ref{theoremLowerBound2} &  by Proposition~\ref{proposition3} &   by Proposition~\ref{proposition5} and Lemma~\ref{lemmaVisitProbability}\\\hline
    \end{tabular}
\end{table}

\section{Conclusions and future work}
\label{secConclusions}
In this paper, we rigorously answer a fundamental question about the fitness level method: what are the tightest lower and upper time bounds that can be constructed using  transition probabilities between fitness levels? 
Drift analysis with fitness levels is developed and the tightest metric bounds from fitness levels are constructed and proven.  
Based on metric bounds, general  linear bounds are constructed.  Coefficients in linear bounds can be calculated recursively or explicitly. Different calculation methods result in different fitness level methods. Drift analysis with fitness levels is a framework that can be used to develop different fitness level methods for different types of time bounds.   Table~\ref{tabBounds} summarizes the  main   bounds discussed in the paper. 

\begin{table}[ht]
   \centering  
 \caption{Type-$c_{k,\ell}$, $c_\ell$ and $c$ bounds. Notation refers to Table~\ref{tab:notation}.}
  \label{tabBounds} \small 
  \begin{tabular}{|l|l|l|}\hline
      {Type }   & $d_k$ is a bound  on the hitting time  $m(X_k)$ &source   \\\hline 
      $r_{k,\ell}$ lower    & $d_k \le  \displaystyle\min_{X_k \in S_k}\left\{ \frac{1}{p(X_k, S_{[0,k-1]})} + \sum^{k-1}_{\ell=1}\frac{p(X_k,S_{\ell})}{p(X_k,S_{[0,k-1]})}  d_{\ell}\right\} $  &Theorems~\ref{theoremLowerBound1},\ref{theoremTightestLowerBound1}  \\ 
      \hline
      $r_{k,\ell}$ upper    & $\displaystyle d_k \ge \max_{X_k\in S_k}\left\{ \frac{1}{p(X_k, S_{[0,k-1]})} + \sum^{k-1}_{\ell=1}\frac{p(X_k,S_{\ell})}{p(X_k,S_{[0,k-1]})}  d_{\ell}\right\}$  &Theorems~\ref{theoremUpperBound1}, \ref{theoremTightestUpperBound1}  \\
      \hline   $c_{k,\ell} $ lower &   $\displaystyle d_k \le \frac{1}{p_{{\scriptscriptstyle\max}}(X_{k},S_{[0,k-1]})} +\sum^{k-1}_{\ell=1} \frac{c_{k,\ell}}{ p_{{\scriptscriptstyle\max}}(X_{\ell},S_{[0,\ell-1]})}$  &Theorem~\ref{theoremLowerBound2} \\ 
      \hline   $c_{k,\ell}  $  upper &   $\displaystyle d_k \ge\frac{1}{p_{{\scriptscriptstyle\min}}(X_{k},S_{[0,k-1]})} +\sum^{k-1}_{\ell=1} \frac{c_{k,\ell}}{ p_{{\scriptscriptstyle\min}}(X_{\ell},S_{[0,\ell-1]})}$ &Theorem~\ref{theoremUpperBound2}  \\
      \hline
        $c_{\ell}$  &   $c_{k,\ell}=c_\ell$, a special case of Type-$c_{k,\ell}$  bounds   &Corollaries~\ref{corollaryLowerBound-cl}, \ref{corollaryUpperBound-cl} 
   \\ \hline $c$    &   $c_{k,\ell}=c$, a special case of Type-$c_{\ell}$  bounds     &Corollaries \ref{corollaryLowerBound-c}, \ref{corollaryUpperBound-c}   
\\ 
   \hline
\end{tabular}
\end{table}

The framework is generic and promising.  It turns out that  Type-$c_{k,\ell}$ bounds are at least as tight as  Type-$c_{\ell}$ and Type-$c$ bounds on any fitness landscapes, and even tighter on fitness landscapes with shortcuts. This is demonstrated by the case study of the (1+1) EA maximizing the TwoMax1 function.    
One direction for future research is to simplify the recursive computation in metric and linear bounds.

\subsection*{Acknowledgments} 
The authors thank Dirk Sudholt for his helpful discussion of the fitness level method, and also Benjamin Doerr and Timo K\"otzing for their kind explanations of their work, which helped to greatly improve the quality of this paper.

\end{document}